\documentclass[11pt]{article}
\usepackage{ifthen}
\usepackage{mdwlist}
\usepackage{amsmath,amssymb,amsfonts,amsthm}
\usepackage{mathtools}
\usepackage{bm}
\usepackage[colorlinks=true,linkcolor=blue,citecolor=blue,urlcolor=blue]{hyperref}
\usepackage{enumitem}
\usepackage{graphicx}
\usepackage{xspace}
\usepackage{verbatim}

\usepackage{algpseudocode}
\usepackage[boxruled,vlined,linesnumbered]{algorithm2e}
	\SetKwFor{While}{While}{}{}
	\SetKwFor{For}{For}{}{}
	\SetKwIF{If}{ElseIf}{Else}{If}{}{elif}{Else}{}
	\SetKw{Return}{Return}
	
\usepackage[margin=1in]{geometry}
\usepackage{color}
\usepackage{thm-restate}
\usepackage{latexsym}
\usepackage{epsfig}
\usepackage[colorinlistoftodos,textsize=scriptsize]{todonotes}

\usepackage{caption}
\usepackage{subcaption}

\def\eps{\ve}
\renewcommand{\epsilon}{\ve}
\def\ve{\varepsilon}

\newcommand{\pr}[2][]{\mathrm{Pr}\ifthenelse{\not\equal{}{#1}}{_{#1}}{}\!\left[#2\right]}

\providecommand{\poly}{\operatorname*{poly}}

\def \cA {{\cal A}}

\def \cS {{\cal S}}

\def \cW {{\cal W}}
\def \cX {{\cal X}}

\newcommand{\RR}{\mathbb{R}}

\newcommand{\ns}{n}
\newcommand{\dist}{\alpha}

\newcommand{\dims}{p}

\newcommand{\absv}[1]{\left|#1\right|}
\newcommand{\norm}[1]{\left\lVert#1\right\rVert_2}
\newcommand{\normz}[1]{\left\lVert#1\right\rVert_0}
\newcommand{\norminf}[1]{\left\lVert#1\right\rVert_{\infty}}

\def \Paren#1{{\left({#1}\right)}}
 
\newcommand{\probof}[1]{\Pr\Paren{#1}}

\newcommand{\expectation}[1]{\mathbb{E}\left[#1\right]}

\makeatletter
\@ifundefined{theorem}{% Theorems (each with its own style)
  \theoremstyle{definition}
  \newtheorem{definition}{Definition}
  \theoremstyle{plain}
  \newtheorem{theorem}{Theorem}
  
  \newtheorem{lemma}{Lemma}

  \theoremstyle{remark}

}{}
\makeatother

\title{Wide Network Learning with Differential Privacy}

\author {
 Huanyu Zhang\\
 Facebook\\
 \tt{huanyuzhang@fb.com}
 \and
 Ilya Mironov\\
 Facebook\\
 \tt{mironov@gmail.com}
 \and
 Meisam Hejazinia\\
 Facebook\\
 \tt{mnia@fb.com}
}

\begin{document}
\maketitle
\begin{abstract}
Despite intense interest and considerable effort, the current generation of neural networks suffers a significant loss of accuracy under most practically relevant privacy training regimes. One particularly challenging class of neural networks are the wide ones, such as those deployed for NLP typeahead prediction or recommender systems. 

Observing that these models share something in common---an embedding layer that reduces the dimensionality of the input---we focus on developing a general approach towards training these models that takes advantage of the \emph{sparsity} of the gradients. 
More abstractly, we address the problem of differentially private empirical risk minimization (ERM) for models that admit sparse gradients.

We demonstrate that for non-convex ERM problems, the loss is logarithmically dependent on the number of parameters, in contrast with polynomial dependence for the general case. Following the same intuition, we propose a novel algorithm for privately training neural networks. Finally, we provide an empirical study of a DP wide neural network on a real-world dataset, which has been rarely explored in the previous work.

% Finally, we provide an empirical study of DP word embedding on a real-world dataset. The experimental results suggest that our method can achieve greater utility while providing a similar privacy guarantee, as compared to the classic DP-SGD algorithm. 

% We complement theoretical analysis of the privacy guarantees of the proposed algorithm with \emph{empirical estimates} of its privacy loss via memorization metrics. Significantly, we demonstrate that these metrics are sufficiently sensitive to inform the choice of parameters for privacy-preserving algorithms.
\end{abstract}

\section{Introduction}

Deep learning models are often trained on datasets that contain sensitive information, such as location, purchase history, or medical records. There is mounting evidence that, in the absence of specific measures to the contrary, deep learning models may memorize and subsequently leak some of their training samples~\cite{fredrikson2015model,carlini2020extracting}.

Differential privacy (DP)~\cite{DworkMNS06} was proposed as a mathematically rigorous definition of privacy, and has since become the gold standard in privacy-preserving machine learning,
% Differential privacy (DP)~\cite{DworkMNS06} allows statistical inference while preserving the privacy of the individual samples.
%DP disclosure methods have found many 
with applications across multiple domains. The notable examples include large technology companies, such as Google, Apple, and Microsoft that rely on differential privacy for privacy-preserving telemetry~\cite{ErlingssonPK14, AppleDP17, DingKY17} and the US Census Bureau, which committed to using differential privacy for its 2020 Census data products.

In this paper, we study the problem of deep learning with differential privacy. More broadly, we consider the problem of DP empirical risk minimization (ERM),
%ERM is one of the most fundamental problems in statistical machine learning, 
which provides a framework of unifying various machine learning models, including regression, support vector machine (SVM), and neural network. The problem can be formulated as follows:
\vspace{-10pt}
\paragraph{ERM:} Given the dataset $D = \{d_1, d_2,\ldots,d_\ns\}$ and a loss function $L$, the goal is to 
\vspace{-5pt}
\[
\text{minimize}~L(w;D) = \frac1{\ns}\sum_{i=1}^{\ns} \ell(w;d_i)~\text{over}~w\in W,
\]
where the map $\ell$ defines a loss function $\ell(\cdot;d)$ over a parameter space $W$ for each data point $d$. 

Our objective is to design a mechanism that satisfies $(\eps,\delta)$-differential privacy (formally defined in Definition~\ref{def:dp}) while minimizing the accuracy loss measured as the empirical risk. This problem is termed DP ERM, which is of vital importance in both academia and industry. 
%For example, wide neural networks are broadly used in ranking and recommendation systems~\cite{MitsoulisHZB19,ChengKHSCAACCI16}. 

This problem has been considered in a line of work starting with Chaudhuri and Hsu~\cite{ChaudhuriH11,BassilyST14,kasiviswanathanJ16, JainT14}. The most general approach to the problem is based on the DP stochastic gradient descent (DP-SGD)~\cite{BassilyST14,AbadiCGMMTZ16}.
For convex and Lipschitz loss functions, %and $W$ contained in the unit $\ell_2$ ball,
Bassily et al.\ shows that DP-SGD achieves the risk of  $\widetilde{\Theta}\Paren{\frac{\poly(\dims)}{\ns}}$.
% where the risk of the algorithm $\cA$ is defined as $R (\cA, D) \coloneqq \expectation{L\Paren{\cA(D);D }}- \min_{w \in W} L(w;D)$. 
% These studies show that this bound cannot be improved in general, even for the squared loss function. 
For non-convex functions, subsequent work demonstrates that $\expectation{\norm{\nabla L\Paren{w^{priv};D }}} \le \widetilde{O}\Paren{\frac{\poly(\dims)}{\sqrt{\ns}}}$~\cite{zhang2017efficient,wang2018differentially,WangJEG19}. In both cases, the risk depends on the dimension $\dims$, in contrast to the non-private setting. This dependency becomes a significant source of accuracy loss for private algorithms, especially when the input dimensionality is large.
%which is prevalent in many applications, e.g., neural networks.

% In~\cite{BassilyST14}, the optimal risk  has been established, which has shown that the risk is $\Theta\Paren{\frac{\poly(\dims)}{\ns}}$, as compared with $\Theta\Paren{\frac1{\ns}}$ in the non-private setting. Note that the number of samples needed to achieve the same risk is no longer independent of the dimension $\dims$. In this case, the privacy has incurred a huge cost,  especially when the dimension size~$\dims$ is huge. 

Other papers have studied the problem of \emph{sparse} DP ERM, which optimizes over sparse \emph{outputs} ($w$ in our notation)~\cite{HardtT10, TalwarTZ15,CaiWZ19}. Under this additional constraint, the loss decreases from $\widetilde{O}\Paren{\frac{\poly(\dims)}{\ns}}$ to $\widetilde{O}\Paren{\frac{\poly(\log \dims)}{\ns}}$. However, this assumption does not hold in many high-dimensional models. For example, in word embedding, the network parameters are evenly important, and artificially reducing the network's dimensionality (such as the size of the input vocabulary) will negatively affect its accuracy.

% In this paper, we consider the problem of $(\eps,\delta)$-Differentially Private Empirical Risk Minimization ($(\eps,\delta)$-DP ERM).

This work pursues the orthogonal direction of decreasing the privacy cost of ERM. Motivated by applications to wide neural networks, which have a wide embedding layer that reduces the dimensionality of the input, we consider the setting where the \emph{input} is inherently sparse. This scenario is common in machine learning tasks. For example, in wide neural networks for language models and recommendation systems~\cite{mitsoulis2019simple,hejazinia2019deep}, the input data is usually extremely sparse, with only a small fraction of parameters ``active'' in each update. Thus, the input sparsity usually produces a sparse gradient, which holds for a broad class of problems, as shown in Section~\ref{sec:erm_2}.
%For example, if we consider the problem of linear regression, 
% the support of the point-wise gradient satisfies $\normz{\nabla(w;(x,y))}\le s$ as long as the input is sparse, i.e., $\normz{x}\le s$, because $\nabla(w;(x,y))=2(w^T x-y)\cdot x$.

The question we address in this work is the following: \textbf{can we utilize the input sparsity (leading to the gradient sparsity) to improve the accuracy of DP deep learning, or broadly, DP ERM?}

In this paper, we provide a positive answer to the above question. Our contributions can be summarized as follows:

\begin{enumerate}
    \item We show that for non-convex ERM with gradient sparsity, the loss is ${O}\Paren{\frac{\poly(\log \dims)}{\ns^{\frac14}}}$, as compared with ${O}\Paren{\frac{\poly( \dims)}{\sqrt{\ns}}}$ for the general case (Section~\ref{sec:dp-erm}).
    
    \item  Based on the theoretical analysis, we propose a novel algorithm that leverages the inherent gradient sparsity in wide neural networks (Section~\ref{sec:nn}).
    
    \item We provide an empirical study of DP training of a wide neural network on a real dataset. The experimental results suggest that our method can achieve a much better utility compared to the standard DP-SGD (Section~\ref{sec:exp}).
    
    \item We complement theoretical analysis of the privacy guarantees with \emph{empirical estimates} of the algorithm's privacy loss via memorization metrics. Significantly, we demonstrate that these metrics are sufficiently sensitive to inform the choice of parameters for privacy-preserving algorithms (Section~\ref{sec:exp}).
        % \item We complement theoretical analysis of the privacy guarantees of the proposed algorithm with \emph{empirical estimates} of its privacy loss via memorization metrics. Significantly, we demonstrate that these metrics are sufficiently sensitive to inform the choice of parameters for privacy-preserving algorithms (Section~\ref{sec:exp}).
\end{enumerate}

\section{Related work}
\subsection{DP neural networks}
The problem of training DP deep learning models was initially addressed by
\cite{ShokriS15}, followed by~\cite{AbadiCGMMTZ16} who proposed DP-SGD to train deep neural networks in a centralized setting. Specifically, Abadi et al.\ clipped each gradient in order to bound the influence of each sample and introduced the moments accountant to track privacy loss. This method has generated substantial interest, and follow-up research, focusing on improving the architecture and applying this approach to different data types, such as images and texts, see, e.g.,~\cite{abay2018privacy, acs2018differentially, beaulieu2019privacy, chen2018differentially, mcmahan2018general, mcmahan2017learning, thakkar2019differentially, wang2019subsampled, papernot2018scalable, RamaswamyTMAMB20, vu2019dpugc,popov2018distributed,lee2020differentially}.

However, despite tremendous interest in privacy-preserving training for deep neural networks, there are only few works that explicitly considers wide neural networks. Perhaps the two most relevant works are~\cite{ShokriS15} and~\cite{zhou2020bypassing}. In particular, \cite{ShokriS15} apply the sparse vector technique to select the subset of the gradients. Their approach is different from ours in the following ways.
First, their primary motivation is the communication cost, not the model accuracy. In contrast, our objective is to utilize the sparsity in wide neural networks and improve the model accuracy. 
Second, Shorki and Shmatikov consider only the sparse vector technique for private selection whereas we allow any DP selection technique. Our experimental results and prior work~\cite{LyuS17} show that the other selection rules (e.g., the exponential mechanism) may outperform the sparse vector techniques. 

\cite{zhou2020bypassing} consider a setting superficially similar to ours, a non-convex DP ERM over large domain and sparse gradients. However, Zhou et al.\ additionally require access to public samples, which we don't assume. Furthermore, Zhou et al.'s theoretical analysis requires roughly the same amount of public and private data, which is hard to get in practice. In addition, their empirical analysis considers training a deep neural network on the MNIST dataset, while we study training a wide neural network on the Brown News corpus, which is a more challenging task.

DP wide neural networks are explicitly or implicitly used in many other previous work, e.g., for word embedding~\cite{vu2019dpugc, RamaswamyTMAMB20,mcmahan2017learning}. However, these studies either assume the embedding is already available (i.e., pre-trained on a large public dataset), or by training a model with DP-SGD or its modified versions, not leveraging sparsity.

%Due to the page limit, we defer the literature of DP-ERM, private selection, and wide neural networks to the supplement.
%Last, our algorithm has another clipping after the selection. In the experiment, we found that it was a critical step, which could boost the model accuracy.

% Although there are quite a few works on the architecture of DP neural networks, people have paid little attention to wide neural networks.  DP wide neural networks

% mainly focused on modify

% and invented the moment accountant to get better privacy loss estimation, which is arguably the most famous work in this area. There are many works 

% There is a massive volume of research over the past decade on designing algorithms for privacy- preserving neural networks. 
\subsection{DP ERM}
DP ERM has been subject of many studies~\cite{ChaudhuriH11,BassilyST14,JainT14, kasiviswanathanJ16, zhang2017efficient, wang2018differentially}.
Particularly,~\cite{BassilyST14} shows that the maximum risk is in the order of  $\widetilde{\Theta}\Paren{\frac{\sqrt\dims}{\ns\eps}}$, for convex and Lipschitz loss functions, and $W$ contained in the unit $\ell_2$ ball. It also shows that this bound cannot be improved in general, even for the squared loss function. For non-convex functions,~\cite{zhang2017efficient, wang2019differentially,  WangJEG19} show that $\expectation{\norm{\nabla L\Paren{w^{priv};D }}} \le \widetilde{O}\Paren{\frac{\dims^{\frac14}}{\sqrt{\ns\eps}}}$.

Other studies~\cite{HardtT10, TalwarTZ15,zhang2017efficient, CaiWZ19} have evaluated the problem of \emph{sparse} DP ERM, with sparse $w$ assumed.  Under this additional constraint, the loss is in the order of $\widetilde{O}\Paren{\frac{\poly(\log \dims)}{\ns}}$ instead of $\widetilde{O}\Paren{\frac{\poly(\dims)}{\ns}}$. 
\subsection{DP selection}
Our algorithm also depends on DP selection as a middle step, of which the goal is to select top $k$ items out of a set of $\dims$ items. There are three classic algorithms for this problem: exponential mechanism~\cite{McSherryT07}, the sparse vector algorithm~\cite{DworkNRRV09, DworkR14}, and report noisy max~\cite{DworkR14}. In~\cite{SteinkeU17b}, the optimal bound has been established for the database setting.~\cite{durfee2019practical} proposes an algorithm which can be unaware of the domain, but may not always output $k$ components. 

\vspace{-10pt}
\section{Preliminaries}

\subsection{Privacy preliminaries}

A \emph{dataset} $X = (X_1,\dots,X_n) \in \cX^n$ is a collection of points from some universe $\cX$. We say that two datasets $X$ and $X^{\prime}$ are neighboring, denoted as $X \sim X^{\prime}$ if they differ in exactly one point. We first provide a formal definition of differential privacy.

\begin{definition}[Differential Privacy (DP)~\cite{DworkMNS06}] 
\label{def:dp}
A randomized algorithm $\cA: \cX^n \rightarrow \cS$ satisfies \emph{$(\eps,\delta)$-differential privacy ($(\eps,\delta)$-DP)} if for every pair of neighboring datasets $X, X' \in \cX^n$, and any event $S \subseteq \cS$, 
\[
\probof {\cA(X) \in S} \leq e^{\eps} \probof{\cA(X^{\prime}) \in S} + \delta.
\]
%The case $\delta=0$ is called \emph{pure differential privacy} and is denoted by $\eps$-DP. 
\end{definition}

 Then we introduce two important properties of differential privacy. The first is the ``sampling property'', which reveals that the privacy level can be boosted by sampling.
 % \begin{lemma}(Privacy amplification via sampling, Theorem 9 in~\cite{BalleBG18})
 % Over a domain of datasets $\cD^{\ns}$, if an algorithm $\cA$ is $(\eps,\delta)$-DP, where $\eps\ge 0$, then for any dataset $D \in \cD^{\ns}$, executing $\cA$ on uniformly random $\gamma \ns$ entries of $D$ ensures $\Paren{log\Paren{1+\gamma(e^\eps-1)}},\gamma \delta)$-DP.
 % \end{lemma}
 \begin{lemma}(Privacy amplification via sampling, Theorem 9 in~\cite{BalleBG18})
 \label{lem:dp_sample}
 Over a domain of datasets $\cX^{\ns}$, if an algorithm $\cA$ is $(\eps,\delta)$-DP, where $\eps \le 1$, then for any dataset $X \in \cX^{\ns}$, executing $\cA$ on uniformly random $\gamma \ns$ entries of $\cX$ ensures $\Paren{\gamma\eps,\gamma \delta}$-DP.
 \end{lemma}

% Another crucial property of differential privacy is that it can be composed adaptively.  By adaptive composition, we mean a sequence of algorithms $\cA_1(X),\dots,\cA_t(X)$ where the algorithm $\cA_t(X)$ may also depend on the outcomes of the algorithms $\cA_1(X),\dots,\cA_{t-1}(X)$.
% \begin{lemma}\label{lem:advcomp}(Strong composition theorem, Theorem 3.2 in~\cite{KairouzOV15})
% For any $\eps>0$, $\delta \in [0,1]$, and $\tilde{\delta} \in [0,1]$, the class of $(\eps,\delta)$-DP mechanisms satisfies $(\tilde{\eps}_{\tilde{\delta}}, k\delta+\tilde{\delta})$-DP under $k$-fold adaptive composition, for
% $$\tilde{\eps}_{\tilde{\delta}}  = k\eps(e^\eps-1)+\eps \sqrt{2k\log\frac{1}{\tilde{\delta}}}.$$
% \end{lemma}
Another critically important property of differential privacy is that it can be composed adaptively. By adaptive composition we mean a sequence of algorithms $\cA_1(X),\dots,\cA_t(X)$ where the algorithm $\cA_t(X)$ may also depend on the outcomes of the algorithms $\cA_1(X),\dots,\cA_{t-1}(X)$.
\begin{lemma}\label{lem:advcomp}(Strong composition theorem, Theorem 3.2 in~\cite{KairouzOV15})
For any $\eps>0$, $\delta \in [0,1]$, and $\tilde{\delta} \in [0,1]$, the class of $(\eps,\delta)$-DP mechanisms satisfies $(\tilde{\eps}_{\tilde{\delta}}, k\delta+\tilde{\delta})$-DP under $k$-fold adaptive composition, for
$$\tilde{\eps}_{\tilde{\delta}}  = k\eps(e^\eps-1)+\eps \sqrt{2k\log\frac{1}{\tilde{\delta}}}.$$
\end{lemma}

 In their original paper,~\cite{DworkMNS06} provides a famous scheme for differential privacy, known as the Gaussian mechanism. This method adds Gaussian noise to a non-private output in order to make it private. We first define the sensitivity, and then state their result. Roughly speaking, the sensitivity measures the maximum difference between the outputs of the algorithm on two neighboring datasets.

 \begin{definition}
 \label{def:sensi}
 The sensitivity of a non-private algorithm $f: \cX^\ns \rightarrow S$ in $\ell_z$ norm is 
 $$s_{z}(f) = \max_{X,X^\prime \text{are neighboring}} |f(X)-f(X^\prime)|_z.$$
 \end{definition}

 \begin{lemma}[Gaussian mechanism~\cite{DworkMNS06}]
 \label{lem:gaussian}
 For any $0 \le \eps \le 1$, $0 \le \delta \le 1$, and $f: \cX^\ns \rightarrow \mathbb{R}^\dims$, $\cA(X) = f(X)+ N (0, \sigma^2 I_{\dims})$ satisfies $(\eps,\delta)$-DP, where $\sigma^2 = \frac{2\log(1.25/\delta)\cdot s_{2}(f)^2}{\eps^2}$.
 \end{lemma}
% We recall another famous $(\eps,\delta)$-DP algorithm, known as the sparse vector technique~\cite{DworkR14}. Intuitively, the algorithm privately selects the largest coordinates (in absolute value) of the input vector $u$, and outputs their noisy versions. 
We recall another foundational DP algorithm, Algorithm~\ref{alg:sv}, known as the sparse vector technique~\cite{DworkNRRV09,DworkR14}.  
Intuitively, the algorithm privately selects the largest coordinates (in absolute value) of the input vector $u$, and outputs their noisy versions.

%\begin{algorithm}
%    \caption{$\textit{NumericSparse}(u(X),\alpha,c_1,\eps,\delta)$}
%   \label{alg:sv}
%    \begin{algorithmic}[1]
%    \STATE {\bfseries Input:} private vector $u(X) \in \RR^{\dims}$, threshold $\alpha$, sparsity parameter $c_1$, privacy parameter $(\eps,\delta)$, sensitivity upper bound $s = s_{\infty} (u(X))$ (as defined in Definition~\ref{def:sensi}).
%    \STATE Let $\eps_1 =0.95 \eps$, $\eps_2 = 0.05 \eps$, and $\sigma(\eps) = \frac{s\cdot \sqrt{32c_1 \log \frac{2}{\delta}}}{\eps}$ 
%    \STATE Let $\hat{\alpha}_0 = \alpha + \text{Lap}\Paren{\sigma(\eps_1)}$, $\textit{count}=0$, and $\hat{u}=0$
%    \FOR {$i=1$ to $\dims$}
%    \STATE  let $v_i = \text{Lap} \Paren{2\sigma\Paren{\eps_1}}$
%     \IF{ $|u_i|+ v_i \ge \hat{\alpha}_\textit{count} $ }
%     \STATE  Let  $\hat{u}_i = u_i + \text{Lap} \Paren{2\sigma(\eps_2)}$
%      \STATE  Let $\textit{count} = \textit{count}+1$
%      \STATE  Let $\hat{\alpha}_\textit{count} = \alpha+ \text{Lap} \Paren{\sigma(\eps_1)}$      
%      \ENDIF
%     \ENDFOR
%     \IF{$\textit{count} \ge c_1$}
%     \STATE break
%     \ENDIF
%     \STATE {\bfseries Output:}  $\hat{u}$
%\end{algorithmic}
%\end{algorithm}

\begin{algorithm}[H]

\KwIn{private vector $u(X) \in \RR^{\dims}$, threshold $\alpha$, sparsity parameter $c_1$, privacy parameter $(\eps,\delta)$, sensitivity upper bound $s = s_{\infty} (u(X))$ (defined in Definition~\ref{def:sensi}).}

Let $\eps_1 =0.95 \eps$, $\eps_2 = 0.05 \eps$, and $\sigma(\eps) = \frac{s\cdot \sqrt{32c_1 \log \frac{2}{\delta}}}{\eps}$ 

Let $\hat{\alpha}_0 = \alpha + \text{Lap}\Paren{\sigma(\eps_1)}$, $count=0$, and $\hat{u}=0$

 \For {$i=1$ to $\dims$}{
 let $v_i = \text{Lap} \Paren{2\sigma\Paren{\eps_1}}$
 
 \If{ $|u_i|+ v_i \ge \hat{\alpha}_{count} $ }
 {Let  $\hat{u}_i = u_i + \text{Lap} \Paren{2\sigma(\eps_2)}$
 
 Let $count = count+1$
 
 Let $\hat{\alpha}_{count} = \alpha+ \text{Lap} \Paren{\sigma(\eps_1)}$
 
 }
 }
  \If{$count \ge c_1 $ }{break}
\KwOut{ $\hat{u}$}
\caption{$\textit{NumericSparse}(u(X),\alpha,c_1,\eps,\delta)$}
\label{alg:sv}
\end{algorithm}

We also provide its theoretical guarantees. Intuitively, the algorithm outputs at most $c_1$ non-zero coordinates, and with high probability, none of the coordinates with value more than $2\alpha$ will be output as zero.

\begin{lemma}\label{lem:sparsevt}
Algorithm~\ref{alg:sv} satisfies $(\eps,\delta)$-DP. Furthermore, let $\dist = \frac{20s \Paren{\log \dims + \log\frac{4c_1}{\beta}} \sqrt{c_1\log \frac{2}{\delta} }}{\eps}$, and $ |\{i\colon |u_i| > 0 \}| \le c_1$. Then with probability at least $1-\beta$, the algorithm does not halt when $i \le \dims$. Furthermore, for all $\hat{u}_i \neq 0$:
$
|\hat{u}_i - u_i| \le \dist,
$
and for all $\hat{u}_i = 0$:
$|u_i|  \le 2\dist.$

\end{lemma}

The following lemma can be viewed as a direct corollary.
\begin{lemma}\label{lem:sparsel2}
Given all the conditions in Lemma~\ref{lem:sparsevt}, with probability at least $1-\beta$, 
    $\norm{u - \hat{u}} \le 2.5\dist\sqrt{c_1}.$
\end{lemma}
\begin{proof}
Let $[\dims] \coloneqq \{1,2,\ldots,\dims\}$, and $A,B$ be subsets of $[\dims]$, with $A\coloneqq \{i\colon |u_i| > 0 \}$ and $B \coloneqq \{i\colon |u_i| = 0 \}$. Note that $|A|\le c_1$, then
% \begin{align*}
%     \norm{u - \hat{u}}^2 &= \sum_{i \in A} (\hat{u}_i - u_i)^2 +\sum_{i \in B}\hat{u}_i^2 \nonumber\\
%     &\le c_1 \dist^2 + c_1 \cdot 4\dist^2 \le 5\alpha c_1.
% \end{align*}
    $\norm{u - \hat{u}}^2 = \sum_{i \in A} (\hat{u}_i - u_i)^2 +\sum_{i \in B}\hat{u}_i^2 
    \le c_1 \dist^2 + c_1 \cdot 4\dist^2 \le 5\alpha^2 c_1$,
where the inequality comes from Lemma~\ref{lem:sparsevt}, and the fact that the algorithm at most outputs $c_1$ non-zero coordinates.
\end{proof}
\subsection{ERM preliminaries}
 We introduce the definition of smooth functions.
 \begin{definition}
We say a function $\ell\colon \mathbb{R}^\dims \rightarrow \mathbb{R}$ is $K$-smooth, if for all $w_1, w_2 \in \mathbb{R}^\dims$,
$$\absv{ f(w_2)-f(w_1) - \langle \nabla f(w_1), w_2-w_1\rangle} \le K \norm{w_2-w_1}^2. $$
 \end{definition}

\section{Improving DP ERM by sparsity}
\label{sec:dp-erm}

In Section~\ref{sec:erm_1}, we provide our theoretical results for $(\eps,\delta)$-DP ERM problems under the assumption of sparse gradients. In Section~\ref{sec:erm_2}, we show that for a broad class of problems, i.e., generalized linear models, sparse input leads to sparse gradients.
\subsection{Private ERM with sparse gradients}
\label{sec:erm_1}
We consider the following empirical risk minimization problem: given a training data set $D$ consisting of $n$ data points $D = \{d_j\}_{j=1}^{\ns}$, where
$d_j \in \RR^p$, a constraint set $\cW \in \RR^p$,
and a loss function $\ell\colon \cW \times \RR^{\dims} \rightarrow \RR$, we
want to find
$w^* = \arg\min_{w\in \cW} ~L(w;D) = \arg\min_{w\in \cW} ~
\frac1\ns{\sum_{j=1}^{\ns} \ell(w; d_j)}$ satisfying differential privacy. 

To characterize the gradient sparsity, we assume the data set has the following structure: $D$ can be evenly divided into $m$ subsets, such that the sum of the gradients of each subset is sparse. Specifically, we let $D = \{ D_1, \ldots,D_m\}$, where $D_1= \{d_1,\ldots,d_{\frac{\ns}{m}}\}, D_2= \{d_{\frac{\ns}{m}+1},\ldots,d_{\frac{2\ns}{m}}\}$, \dots, and $D_{m}= \{d_{\frac{(m-1)\ns}{m}+1},\ldots,d_{\ns}\}$, such that
\begin{align}
\label{equ:group_sparsity}
\forall i \in [m],~~ \normz{\sum_{d_j \in D_i} \nabla \ell(w;d_j)} \le c_1,~\text{for all}~ w \in W.
\end{align}

Roughly speaking, this assumption requires that the original dataset can be partitioned into several parts, so that each part exhibits some sparsity similarity.

This assumption may appear overly strict. However, we justify that it can be satisfied in many real applications. For example, in recommender systems, samples collected from the same user usually have overlapping supports, leading to input sparsity \cite{hu2017mitigating}. In these scenarios, it is natural to partition the input dataset according to the user ID. %Similarly, in NLP, training samples are collected from different sources, where each source usually exhibits some similarity. 
Similarity also exists in NLP, where training samples are collected from different sources.
In Section~\ref{sec:erm_2}, we show that for a broad class of problems, sparse input always produces sparse gradients.

We note that the sparsity assumption is used to argue the \emph{utility} guarantee of our algorithm for ERM problems, without being necessary for its \emph{privacy}.

To this end we propose Algorithm~\ref{alg:privatelr}. Intuitively, in each iteration the algorithm first selects the most competitive coordinates of the gradient, and only adds noise to them. Finally it updates the model according to the noisy version of the gradient. We provide the privacy guarantee and theoretical guarantees in Theorems~\ref{thm:ermpri} and~\ref{thm:erm}.

%
%\begin{algorithm}
%    \caption{Differentially private ERM}
%   \label{alg:privatelr}
%    \begin{algorithmic}[1]
%    \STATE {\bfseries Input:}  Data set $D = \{d_j\}_{j=1}^{\ns} $, loss function $L(w;D)$, privacy parameters $(\eps,\delta)$, constraint set $\cW$, learning rate $\eta_t$, iteration times $T$, and gradient $\ell_\infty$-norm bound $\norminf{\nabla_{\ell(w;d)}}\le c_2$
%    \STATE  Initialize $w_0$ from an arbitrary point in $\cW$
%    \FOR{$t=0$ \text{to} $T-1$}
%    \STATE Pick $D_i \sim_u D$ with replacement 
%      \STATE Compute its average gradient $\nabla_t = \frac{m}{\ns}\sum_{d_j \in D_i} \nabla \ell(w;d_j)$
%      \STATE Let $\Delta_t = \textit{NumericSparse}(\nabla_t, \alpha, c_1, \epsilon^{\prime}, \delta^{\prime})$, where $\dist = \frac{40c_2m \Paren{\log \dims + \log (4c_1 \ns)} \sqrt{c_1\log \frac{2}{\delta^\prime} }}{\ns\eps^\prime}$, $\epsilon^{\prime} = \frac{\epsilon\cdot m}{2\sqrt{2T \log \frac{2}{\delta}}}$ and $\delta^{\prime} = \frac{\delta m}{2T}$
%      \STATE  $w^{t+1} = w_t - \eta_t \Delta_t$    
%     \ENDFOR
%\end{algorithmic}
%\end{algorithm}

\begin{algorithm}[H]

\KwIn{Data set $D = \{d_j\}_{j=1}^{\ns} $, loss function $L(w;D)$, privacy parameters $(\eps,\delta)$, constraint set $\cW$, learning rate $\eta$, iteration times $T$, and gradient $\ell_\infty$-norm bound $\norminf{\nabla{\ell(w;d)}}\le c_2$}

Initialize $w_0$ from an arbitrary point in $\cW$

 \For {$t=0$ \textrm{to} $T$}{
 Pick $D_i \sim_u D$ with replacement 
 
Compute its average gradient $\nabla_t = \frac{m}{\ns}\sum_{d_j \in D_i} \nabla \ell(w_t;d_j)$
 
 Let $\Delta_t = \textit{NumericSparse}(\nabla_t, \alpha, c_1, \epsilon^{\prime}, \delta^{\prime})$, where $\dist = \frac{40c_2m \Paren{\log \dims + \log (4c_1 \ns)} \sqrt{c_1\log \frac{2}{\delta^\prime} }}{n\eps^\prime}$, $\epsilon^{\prime} = \frac{\epsilon\cdot m}{2\sqrt{2T \log \frac{2}{\delta}}}$ and $\delta^{\prime} = \frac{\delta m}{2T}$ \
 
 $w_{t+1} = w_t - \eta \Delta_t$\
 }
\KwOut{ $w^{priv} = w_T$}

\caption{Differentially private ERM with sparse gradients }
\label{alg:privatelr}
\end{algorithm}

\begin{theorem}[Privacy]
\label{thm:ermpri}
With the assumption that $m\le 10\sqrt{T}$, Algorithm~\ref{alg:privatelr} satisfies $(\eps,\delta)$-DP. 
\end{theorem}
\smallskip
 \begin{proof}
 We note that in each iteration, when fixing the randomness due to sampling, step 4 itself satisfies $( \epsilon^{\prime}, \delta^{\prime})$-DP, where $\epsilon^{\prime} = \frac{\epsilon\cdot m}{2\sqrt{2T \log \frac{2}{\delta}}}$ and $\delta^{\prime} = \frac{\delta m}{2T}$. Then by the sampling property of differential privacy (Lemma~\ref{lem:dp_sample}), each iteration ensures $(\frac{\epsilon}{\sqrt{2T \log \frac{2}{\delta}}}, \frac{\delta}{2T})$-DP. To conclude the proof, we apply the ``strong composition theorem'' (Lemma~\ref{lem:advcomp}) with $\tilde{\delta} = \frac{\delta}{2}$ and $k=T$.
 \end{proof}

\begin{theorem}[Utility]
\label{thm:erm}
We assume $\forall d,~\ell(w;d)$ is $K$-smooth; $\forall w,d$, $\norminf{\nabla{\ell(w;d)}}\le c_2$ and $\norm{\nabla{\ell(w;d)}}\le G$. Furthermore, under Assumption \eqref{equ:group_sparsity},
\vspace{-8pt}
\begin{enumerate}
\item If $L(w^0;D) - L(w^*;D) \le D_{L}$, and we set $T = \max \Paren {\frac{m^2}{100},\ns}$, then
%then with probability at least $1-\beta$ over the randomness of the algorithm,
\begin{align}
\expectation{\frac1{T} \sum_{t=1}^T \norm{\nabla L (w_t)}^2} = \widetilde{O}\Paren{\frac{G\Paren{c_1c_2+ \sqrt{K D_{\ell } }}} {\eps}  \Paren{\frac{m}{n}+\frac1{\sqrt{\ns}} }}. \nonumber
\end{align}
%Furthermore, if $\forall d$, $\ell(w;d)$ is convex,
%\begin{align}
%\min \Paren{ L(w_t;D) - L(w^*;D)  } = \widetilde{O}\Paren{ \frac{ c_1 c_2 D_{L} }{\ns \eps}+\frac{ c_1 c_2^{1.5} \sqrt{ D_{\ell } }    } {\sqrt{\ns\eps L}}  }. \nonumber
%\end{align}
\item
Assume that for all $d$,  $\ell(w;d)$ is convex in $w$, and $\forall t$, $\norm{w_t - w^*} \le D_w$. Let $T = \max \Paren {\frac{m^2}{100},\ns}$, then
\begin{align}
\expectation{\frac1T \sum_{t=1}^{T} \Paren{ L(w_t;D) - L(w^*;D) }}=\widetilde{O}\Paren{\frac{D_w\Paren{G+c_1c_2}} {\eps} \cdot \Paren{\frac{m}{n}+\frac1{\sqrt{\ns}} }}. \nonumber
\end{align}

\end{enumerate}

\end{theorem}

 \begin{proof}

 First by the privacy guarantees of the sparse vector technique (Lemma~\ref{lem:sparsel2}), for each iteration with probability greater than $1-\frac1{n}$,
 \begin{align}
 &\norm{\nabla_t - \Delta_t} \le \frac{100c_1 c_2 \sqrt{T}\Paren{\log \dims + \log(4c_1\ns) } \log \frac{T}{m\delta}}{\ns\eps}, \nonumber
 \end{align}
where we remark that the sensitivity upper bound is $s_\infty = \frac{2c_2m}{\ns}$.

Note that from the assumption $\norm{\nabla_t - \Delta_t} \le 2G$ for sure. Therefore, 
 \begin{align}
 \expectation{\norm{\nabla_t - \Delta_t}} &\le \frac{100c_1 c_2 \sqrt{T}\Paren{\log \dims + \log(4c_1\ns) } \log \frac{T}{m\delta}}{\ns\eps}+\frac{G}{\ns}\nonumber \\
 &\le \frac{200 c_1 c_2 \sqrt{T}\Paren{\log \dims + \log(4c_1\ns) } \log \frac{T}{m\delta}}{\ns\eps}\nonumber.
 \end{align}
 where the second inequality comes from the fact that $T\ge 1$, and the second term dominates.
 
 Now we need the following lemma. The first half comes from~\cite{AgarwalSYKM18}, and we prove the second half of the lemma in Appendix~\ref{app:proof}.
 %\begin{lemma}(Corollary 1 in~\cite{AgarwalSYKM18})
 \begin{lemma}
 \label{lem:theer}
 Suppose $\forall d$, $\ell(w;d)$ is $K$-smooth, with $\norm{\nabla{\ell(w;d)}}\le G$. Let $w^0$ satisfy $L(w^0;D) - L(w^*;D) \le D_{L}$. Let $\eta \coloneqq \min \Paren{ \frac1K, \sqrt{2D_L} \Paren{\sigma \sqrt{KT}}^{-1} }$, then after $T$ rounds,
 $$
 \expectation{\frac1{T} \sum_{t=1}^T \norm{\nabla L (w_t)}^2} \le \frac{2D_{L} K}{T} + \frac{2\sqrt{2} \sigma \sqrt{K D_{L}}} {\sqrt{T}}+ GB.
 $$
Besides, if we further assume $\forall d$, $\ell(w;d)$ is convex, and $\forall t \in [T]$, $\norm{w_t - w^*} \le D_w$. Let $\eta \coloneqq \min \Paren{ \frac1K, D_w \Paren{\sigma \sqrt{T}}^{-1} }$, then after $T$ rounds.
 \begin{align}
 \expectation{\frac1T  \sum_{t=1}^T \Paren{ L(w_t;D) - L(w^*;D) }}  \le \frac{D_w^2 K}{T} + \frac{D_w \sigma}{\sqrt{T}}+ 2B D_w\Paren{1 + \frac{G}{\sigma \sqrt{T}}}, \nonumber
 \end{align}
 where 
 \begin{align}
 &\sigma^2 = 2\max_{1\le t \le T} \expectation{\norm{\nabla_{t} -\nabla L(w_t;D)}^2 } + 2\max_{1\le t \le T} \expectation{\norm{\nabla_{t} - \Delta_t}^2}, \text{and} \nonumber\\ 
 &B^2= 2\max_{1\le t \le T} \expectation{\norm{\nabla_{t} - \Delta_t}^2}. \nonumber
 \end{align}
 \end{lemma}

 %Note that $\ell$ is $(\sqrt{c_1}c_2)$-Lipschitz, since $\norm{\nabla_{\ell(w;d)}} \le \sqrt{c_1}c_2$. 

First we consider the non-convex setting. By the definition of $B$ and $\sigma$ in Lemma~\ref{lem:theer}, we have $B = \frac{300 c_1 c_2 \sqrt{T}\Paren{\log \dims + \log(4c_1\ns) } \log \frac{T}{m\delta}}{\ns\eps}$, and $\sigma^2 = 2B^2 + 2G^2$. Therefore,
 \begin{align}
 \expectation{\frac1{T} \sum_{t=1}^T \norm{\nabla L (w_t)}^2 } &\le \frac{2D_{L} K }{T}+ \frac{4\sqrt{2} (B+G) \sqrt{K D_{\ell }}} {\sqrt{T}} + GB.  \nonumber
 \end{align}
 %We take $T = \min \Paren {\frac{m^2}{100},\ns}$. 
Suppose $m<10\sqrt{\ns}$, where $T=\max\Paren{\frac{m^2}{100},\ns} = n$. By Lemma~\ref{lem:theer}, 
 \begin{align}
 \expectation{\frac1{T} \sum_{t=1}^T \norm{\nabla L (w_t)}^2}  &= O\Paren{ \frac{ D_{L} K }{\ns}+ \frac{ G \sqrt{K D_{\ell }}} {\sqrt{\ns}} +\frac{ c_1 c_2 G  \Paren{\log \dims + \log(4c_1\ns) } \log \frac{m}{\delta}  } {\sqrt{\ns}\eps}  }\nonumber\\
 & = \widetilde{O}\Paren{\frac{G\Paren{c_1c_2+ \sqrt{K D_{\ell } }}    } {\sqrt{\ns}\eps}  }. \nonumber
 \end{align}
We then consider the case when $m\ge 10\sqrt{\ns}$, where $T={\frac{m^2}{100}}$. Note that $\frac1{m} \le \frac{m}{100n}$.
 \begin{align}
 \expectation{\frac1{T} \sum_{t=1}^T \norm{\nabla L (w_t)}^2}  &= O\Paren{ \frac{ D_{L} K }{m^2}+ \frac{ G \sqrt{K D_{\ell }}} {m} +\frac{ c_1 c_2 m G  \Paren{\log \dims + \log(4c_1\ns) } \log \frac{m}{\delta}  } {\ns\eps}  }\nonumber\\
 & = \widetilde{O}\Paren{\frac{Gm\Paren{c_1c_2+ \sqrt{K D_{\ell } }}    } {n\eps}  }. \nonumber
 \end{align}
 Therefore, 
 \begin{align}
 \expectation{\frac1{T} \sum_{t=1}^T \norm{\nabla L (w_t)}^2} = \widetilde{O}\Paren{\frac{Gm\Paren{c_1c_2+ \sqrt{K D_{\ell } }}} {n\eps}  }+\widetilde{O}\Paren{\frac{G\Paren{c_1c_2+ \sqrt{K D_{\ell } }}} {\sqrt{\ns}\eps}}, \nonumber
 \end{align}
 and we have proved the first part of Theorem~\ref{thm:erm}.

 Similarly, for convex loss functions,
 \begin{align}
 \expectation{\frac1T  \sum_{t=1}^T\Paren{ L(w_t;D) - L(w^*;D) }  } &\le \frac{D_w^2 K}{T} + \frac{D_w (B+G )}{\sqrt{T}}+ 2B D_w\Paren{1 + \frac{G}{(B+G)\sqrt{T}}}. \nonumber\\
 &\le \frac{D_w^2 K}{T} + \frac{D_w G}{\sqrt{T}}+ 2B D_w. \nonumber
 \end{align}
 If we take $T =\max(\frac{m^2}{100},\ns)$, and by similar arguments,
 % \begin{align}
 % \frac1T  \expectation{\sum_{t=1}^T \Paren{ L(w_t;D) - L(w^*;D) }}  &= O\Paren{  \frac{D_w^2 K \sqrt{c_1}}{\ns\eps} + \frac{D_w c_1^{\frac{3}{4}} c_2 }{\sqrt{\ns\eps}}+ \frac{2D_w c_1^{\frac{3}{4}} c_2  \Paren{\log \dims + \log(4c_1\ns) } \log \frac{T}{\ns\delta}}{\sqrt{\ns\eps}}  } \nonumber\\
 % &= \widetilde{O}\Paren{ \frac{D_w^2 K \sqrt{c_1}}{\ns\eps} + \frac{D_w c_1^{\frac{3}{4}} c_2 }{\sqrt{\ns\eps}}}.\nonumber
 % \end{align}
 \begin{align}
 \expectation{\frac1T \sum_{t=1}^T \Paren{ L(w_t;D) - L(w^*;D) }}  &=  \widetilde{O}\Paren{ \frac{D_w (G+c_1c_2)}{\sqrt{\ns}\eps} + \frac{D_w m (G+c_1c_2)}{n\eps}}.\nonumber
 \end{align}

 %For the second part, we need the following property of convex smooth functions:
 %
 %\begin{lemma}
 %\label{lem: property_smooth_convex}
 %Let $\ell$ be a convex and $L$-smooth function on $\mathbb{R}^{\dims}$, 
 %$$f(y) - f(x) - \nabla f (x)^ T (y-x) \ge \frac1{2L} \norm{\nabla f(y) - \nabla f(x)}^2 , ~~~ \forall x,y \in \mathbb{R}^{\dims}.$$\end{lemma}
 %
 %Let $w^{t^\prime} = \arg\min {\ell(w_t;D)}$, by Lemma~\ref{lem: property_smooth_convex},
 %
 %
 %\begin{align}
 %\min \Paren{ L(w_t;D) - L(w^*;D)  } = \widetilde{O}\Paren{ \frac{ c_1 c_2 D_{L} }{\ns \eps}+\frac{ c_1 c_2^{1.5} \sqrt{ D_{\ell } }    } {\sqrt{\ns\eps L}}  }. \nonumber
 %\end{align}

 \end{proof}

\subsection{Sparse features lead to sparse gradients}
\label{sec:erm_2}
In this section, we show that for generalized linear model (GLM), sparse input always leads to sparse gradients. 

%GLM is one of the most fundamental models in statistics and machine learning, which is introduced as a way of unifying various statistical models, including linear, logistic and Poisson regressions. We define GLM as follows:

\smallskip

\noindent\textbf{GLM:}
Let a dataset $D=(x_j, y_j)_{j=1}^\ns$,  where $ \forall j, x_j \in \mathbb{R}^\dims$, and $y_j \in [0, 1]$. Let $\Phi\colon \mathbb{R} \rightarrow \mathbb{R}$ be a \textit{cumulative generating function}. The objective of GLM is to minimize $\frac{1}{n}\sum_{i=1}^n[\Phi(\langle x_i, w\rangle)-y_i\langle x_i, w\rangle]$.

%Let $y\in [0, 1]$ be the response variable that belongs to an exponential family with natural parameter $\eta$. 
% That is, its probability density function can be written as $p(y|\eta)= \exp(\eta y-\Phi(\eta))h(y)$, where $\Phi$ is the \textit{cumulative generating function}. Given observations $y^1, \cdots, y_n$ such that $y_i\sim p(y_i|\eta_i)$ for $\eta=(\eta_1, \cdots, \eta_n)$, the maximum likelihood estimator (MLE) can be written as  $p(y^1,y^2,\cdots|\eta)=\exp(\sum_{i=1}^n y_i\eta_i-\Phi(\eta_i))\Pi_{i=1}^nh(y_i)$. In GLM, we assume that $\eta$ is modeled by linear relations, {\em i.e.,}
%$\eta_i=\langle x_i,  w^*\rangle$ for some $w^*\in \mathbb{R}^p$ and feature vector $x_i$. Thus, maximizing MLE is equivalent to minimizing $\frac{1}{n}\sum_{i=1}^n[\Phi(\langle x_i, w\rangle)-y_i\langle x_i, w\rangle]$.

% The goal is to find $w^*$, which is equivalent to minimizing its population version  
% \begin{equation}\label{eq:1}
%     w^*=\arg\min_{w\in \mathbb{R}^p}\mathbb{E}_{(x,y)}[\Phi(\langle x, w\rangle )- y\langle x, w \rangle ].
% \end{equation}
In the following lemma, we observe that for GLM, sparse input produces sparse gradients. 
\begin{lemma}
For all $x \in \mathbb{R}^\dims$ with $\normz{x}\le c_1$, and $y \in [0,1]$,
$$\normz{\nabla \ell(w;(x,y))} \le c_1.$$
\end{lemma}

\begin{proof}
Note that $\ell(w;(x,y)) = \Phi(\langle x_i, w\rangle)-y_i\langle x_i, w\rangle $, and 
$\nabla \ell(w;(x,y)) = \Paren{\Phi^\prime (\langle x_i, w\rangle)- y_i} \cdot x_i$.  Therefore, $\normz{\nabla \ell(w;(x,y))} \le c_1.$
\end{proof}

As a corollary, for a group of samples $(x_i,y_i)_{i=1}^{\frac{\ns}{m}}$, assuming the non-zero coordinates are the same for each $x_i$ gives the condition in Equation~\eqref{equ:group_sparsity}.
%the average gradient is sparse suppose they share the same set of input features.

%\input{est-ub-new}

\section {Improving DP-SGD in neural networks}
\label{sec:nn}

In this section, we move to a specific problem of privately training neural networks, which is arguably the most important application of DP-ERM. \cite{AbadiCGMMTZ16} put forward the DP-SGD algorithm, which has been explored in a variety of domains such as federated learning of language models~\cite{mcmahan2017learning} or sharing of clinical data~\cite{beaulieu2019privacy}. However, DP-SGD suffers from a loss in accuracy compared to its non-private version, especially for smaller datasets and high-dimensional networks~\cite{BagdasaryanPS19}.
%\todo{any other good reference?}

Following the observations from the previous sections, a natural question is how to improve DP-SGD for tasks that exhibit input sparsity, which are ubiquitous---and practically important---in domains where neural network models excel. For example, in language models, the first layer of the neural network is usually an embedding layer, whose input is extremely sparse. Accordingly, only a tiny fraction of parameters are picked up and updated in each round of training.

For models with sparse inputs, applying DP-SGD can lead to a poor performance, since the noise has to be added to all the dimensions. However, we cannot directly apply Algorithm~\ref{alg:privatelr} because of the following two reasons. First, there is no upper bound of $\normz{\nabla \ell}$ or $\norminf{\nabla \ell}$. Second, Algorithm~\ref{alg:privatelr} has to aggregate mini-batches according to the feature similarity, which is impractical for training large networks. In this section, we develop a modification of the previous algorithm to handle wide neural networks,  as outlined in Algorithm~\ref{alg:privatenn}.

%\begin{algorithm}
%\caption{DP Sparse}
%\label{alg:privatenn}
%    \begin{algorithmic}[1]
%    \STATE {\bfseries Input:} Data set $D = \{d_1,\ldots, d_\ns\}$, loss function $L(w;D)= \frac1{\ns}\sum_{j=1}^{\ns} \ell (w;d_j)$, where $w \in \mathbb{R}^{\dims}$,  $(\eps^\prime,\delta^\prime)$-DP selection algorithm $M\colon \mathbb{R}^{\dims}  \rightarrow  \{0,1\}^{\dims}$, parameters: learning rate $\eta_t$, noise multiplier $\sigma$, mini-batch size $b$, sparsity parameter $\gamma$, gradient norm bound $S_1$, $S_2$
%    \STATE  Initialize $w_0$ randomly
%    \FOR{$t=0$ \text{to} $T-1$}
%    \STATE Form random batch $b_t$ with sampling probability $b/\ns$
%      \STATE For each $d_j \in b_t$, compute $g_j = \nabla \ell(w_t;d_j)$
%      
%      $\rhd$ The first gradient clipping
%      \STATE $\hat{g}_j = g_j/ \max \Paren{1, \frac{\norm{g_j}}{S_1}}$, and $\hat{g} = \frac1{b}\sum_{d_j \in b_t}\hat{g}_j $
%      
%      $\rhd$ Private selection
%      \STATE Let $M = \cA (\hat{g},\gamma)$, and $\Delta = M \odot \hat{g}$, where $\odot$ is the Hadamard product
%     
%    $\rhd$ The second gradient clipping
%	\STATE $\hat{\Delta} =\Delta/ \max \Paren{1, \frac{\norm{\Delta}}{S_2}}$
%	
%	$\rhd$ Noise addition and parameter update
%	\STATE $\widetilde{\Delta} = \hat{\Delta}+ N\Paren{0, \sigma^2 \min\Paren{\frac{S_1^2}{b^2}, S_2^2 }\cdot \mathbb{I}}$
% 	\STATE $w^{t+1} = w_t - \eta_t \Paren{ \widetilde{\Delta} \odot M}$
%     \ENDFOR
%      \STATE {\bfseries Output:} $w^\textit{priv} = w_T$
%\end{algorithmic}
%\end{algorithm}

\begin{algorithm}[H]

\KwIn{Data set $D = \{d_1,\ldots,d_\ns\}$, loss function $L(w;D)= \frac1{\ns}\sum_{j=1}^{\ns} \ell (w;d_j)$, where $w \in \mathbb{R}^{\dims}$,  $(\eps^\prime,\delta^\prime)$-DP selection algorithm $M: \mathbb{R}^{\dims}  \rightarrow  \{0,1\}^{\dims}$, parameters: learning rate $\eta$, noise multiplier $\sigma$, mini-batch size $b$, sparsity parameter $\gamma$, gradient norm bound $S_1$, $S_2$ }
%  privacy parameters $(\eps,\delta)$, constraint set $\cW$, learning rate $\eta_t$, and iteration times $T$}

Initialize $w_0$ randomly

 \For {$t=0$ to $T-1$}{
Take a random batch $b_t$ with sampling probability $b/\ns$
 
For each $d_j \in b_t$, compute $g_j = \nabla \ell(w_t;d_j)$

$\rhd$ The first gradient clipping

$\hat{g}_j = g_j/ \max \Paren{1, \frac{\norm{g_j}}{S_1}}$, and $\hat{g} = \frac1{b}\sum_{d_j \in b_t}\hat{g}_j $

$\rhd$ Private selection

Let $M = \cA (\hat{g},\gamma)$, and $\Delta = M \odot \hat{g}$, where $\odot$ is the Hadamard product

$\rhd$ The second gradient clipping

$\hat{\Delta} =\Delta/ \max \Paren{1, \frac{\norm{\Delta}}{S_2}}$

$\rhd$ Noise addition and parameter update:

$\widetilde{\Delta} = \hat{\Delta}+ N\Paren{0, \sigma^2 \min\Paren{\frac{S_1^2}{b^2}, S_2^2 }\cdot \mathbb{I}}$
 
$w^{t+1} = w_t - \eta \Paren{ \widetilde{\Delta} \odot M}$
}
\KwOut{ $w^{priv} = w_T$}

\caption{Differentially private optimization with sparse gradients }
\label{alg:privatenn}
\end{algorithm}

\paragraph{The first gradient clipping:} Similarly to DP-SGD, our algorithm requires a bounded influence of each individual sample. We clip each gradient in the $\ell_2$ norm: i.e., the gradient $g_j$ is replaced by $\hat{g}_j = g_j/ \max \Paren{1, \frac{\norm{g_j}}{S_1}}$, which ensures that if $\norm{g_j}\ge S_1$, then $\norm{g_j}$ is scaled down to $S_1$, else its norm is preserved. Then we aggregate the gradient of each sample and compute $\hat{g}$, which is the average gradient of the batch.

\paragraph{Private selection:} This is a new step that specifically targets input sparsity. Its objective is to select the most ``competitive'' coordinates from $\hat{g}$, which will be updated in the current iteration. It is a key step in our algorithm, since it avoids adding too much noise to the parameters. In this step, $M$ is a binary vector, indicating which coordinate is selected, with $\normz{M}/ \dims \approx \gamma$. $\cA$ can be any differentially private selection algorithm, such as the sparse vector technique, exponential mechanism~\cite{DworkR14}, or the algorithm proposed in~\cite{DurfeeR19}. It is likely that another clipping is necessary, depending on the output of the private selection algorithm. We describe the exponential mechanism in Algorithm~\ref{alg:exp-selection} as an example.

%\begin{algorithm}[hbtp]
%\caption{$(\eps^\prime, \delta^\prime)$-DP selection with exponential mechanism }
%\label{alg:exp-selection}
%    \begin{algorithmic}[1]
%    \STATE {\bfseries Input:} input gradient $\hat{g} \in \mathbb{R}^\dims$, sparsity parameter $\gamma$, privacy parameter $\eps^\prime$, $\delta^\prime$, gradient $\ell_\infty$-norm bound $S_0$
%    \STATE  Initialize $M = 0$.
%    \FOR{$t$ \textrm{in} $[\dims]$}
%    \STATE $\hat{g}(t) = \hat{g}(t)/ \max \Paren{1, \frac{\norm{\hat{g}(t)}}{S_0}}$
%    \ENDFOR
%     \FOR {$t=0$ \textrm{to} $\lfloor \gamma \dims \rfloor$}
%    \STATE Randomly draw dimension $k$ with probability proportional to $\exp\Paren{\frac{\eps^{\prime\prime} \absv{\hat{g}(k)}}{2S_0}}$, where $\eps^{\prime\prime}=\frac{\eps^\prime}{ \sqrt{2 \lfloor \gamma \dims \cdot \rfloor \log\frac{1}{\delta^\prime}}}$
%      \STATE $M_k = 1$, $\hat{g}(k)= -\infty$
%       \ENDFOR
%      \STATE {\bfseries Output:} $M \in \{0,1\}^\dims$
%\end{algorithmic}
%\end{algorithm}

\begin{algorithm}[H]

\KwIn{Input gradient $\hat{g} \in \mathbb{R}^\dims$, sparsity parameter $\gamma$, privacy parameter $\eps^\prime$, $\delta^\prime$, gradient $\ell_\infty$ norm bound $S_0$}

Initialize $M = 0$.

 \For {$t$ \textrm{in} $[\dims]$}{
$\hat{g}(t) = \hat{g}(t)/ \max \Paren{1, \frac{\norm{\hat{g}(t)}}{S_0}}$
}

 \For {$t=0$ \textrm{to} $\lfloor \gamma \dims \rfloor$}{
 
Randomly draw dimension $k$ with probability proportional to $\exp\Paren{\frac{\eps^{\prime\prime} \absv{\hat{g}(k)}}{2S_0}}$, where $\eps^{\prime\prime}=\frac{\eps^\prime}{ \sqrt{2 \lfloor \gamma \dims \cdot \rfloor \log\frac{1}{\delta^\prime}}}$

$M_k = 1$, $\hat{g}(k)= -\infty$
 }

\KwOut{$M \in \mathbb{R}$}

\caption{$(\eps^\prime, \delta^\prime)$-DP selection with exponential mechanism }
\label{alg:exp-selection}
\end{algorithm}

\medskip 

\noindent \textbf{The second gradient clipping:} We remark that $\norm{\Delta}$ is usually much smaller than $\norm{\hat{g}}$, because of the impact of the private selection procedure. Applying the second gradient clipping, we can further reduce the amount of the noise, i.e., the standard deviation of the Gaussian noise. We note that this step is necessary, since we observed that it had a significant influence on the algorithm's performance in the experiments. 

\medskip

\noindent \textbf{Noise adding and parameter update:} We first remark that $\norm{\hat{\Delta}} \le  S_2$, so the $\ell_2$-sensitivity of $\hat{\Delta}$ is upper bounded by $2\cdot \min\Paren{\frac{S_1}{b}, S_2}$. Second, since we have already picked up the coordinates to be updated, it is no longer necessary to add noise to all the coordinates. Instead, the noise is only added to the dimensions which are chosen by the private selection algorithm.

%\medskip

Finally, we give the formal privacy guarantee of our algorithm, where we defer the proof to the supplement.
% \begin{theorem}
% \label{the:neural}
% Algorithm~\ref{alg:privatenn} satisfies $\Paren{ \frac{4b \sqrt{T\log\Paren{\frac{n}{2bT\delta^\prime}}}}{\ns} \cdot \Paren{ \eps^\prime+ \frac{\sqrt{2\log\Paren{1.25/\delta^\prime}}}{\sigma}} , \frac{4bT\delta^\prime}{\ns}}$-DP, with the assumption that $ \frac{b}{\ns} \cdot \Paren{ \eps^\prime+ \frac{\sqrt{2\log\Paren{1.25/\delta^\prime}}}{\sigma}} \le \frac{1}{\sqrt{T}}$.
% \end{theorem}
\begin{theorem}
\label{the:neural}
With the assumption that $ \frac{b}{\ns} \Paren{ \eps^\prime+ \frac{2\sqrt{2\log\Paren{1.25/\delta^\prime}}}{\sigma}} \le \frac{1}{\sqrt{T}}$,

Algorithm~\ref{alg:privatenn} satisfies $\Paren{ \frac{4b \sqrt{T\log\Paren{\frac{n}{2bT\delta^\prime}}}}{\ns} \cdot \Paren{ \eps^\prime+ \frac{2\sqrt{2\log\Paren{1.25/\delta^\prime}}}{\sigma}} , \frac{4bT\delta^\prime}{\ns}}$-DP.
\end{theorem}
\begin{proof}
In each iteration, there are two steps which incur privacy costs: private selection and noise addition.
Note that the noise adding satisfies $(\frac{2\sqrt{2\log\Paren{1.25/\delta^\prime}}}{\sigma},\delta^\prime)$-DP, by the privacy guarantee of the Gaussian mechanism (Lemma~\ref{lem:gaussian}).
Then by the sampling theorem (Lemma~\ref{lem:dp_sample}) and the composition theorem, each iteration satisfies $(\tilde{\eps},\tilde{\delta})$-DP, where $\tilde{\eps} = \frac{b}{\ns} \cdot \Paren{ \eps^\prime+ \frac{2\sqrt{2\log\Paren{1.25/\delta^\prime}}}{\sigma}},\tilde{\delta} = \frac{2b\delta^\prime}{\ns}$. Finally, by the strong composition theorem (Lemma~\ref{lem:advcomp}), the algorithm satisfies $(\eps,\delta)$-DP, where
$$\eps  = T\tilde{\eps}(e^{\tilde{\eps}}-1)+\tilde{\eps} \sqrt{2T \cdot \log\frac{1}{T\tilde{\delta}}}, ~\delta = 2T\tilde{\delta}.$$
By the assumption that $ \frac{b}{\ns} \cdot \Paren{ \eps^\prime+ \frac{2\sqrt{2\log\Paren{1.25/\delta^\prime}}}{\sigma}} \le \frac{1}{\sqrt{T}}$, $\eps \le 2\tilde{\eps} \sqrt{2T\log\frac{1}{T\tilde{\delta}}}$.
\end{proof}

\paragraph{Remark I:} The assumption in the theorem is very weak. For example, if we assume the sample rate $\frac{b}{\ns}=\frac1{10000}$, 
each iteration's privacy cost is $\frac1{2000}$, the algorithm has to run for $400$ epochs to violate the assumption!

\paragraph{Remark II:} Except for the private selection, our privacy guarantee is roughly $\sqrt{\log \frac{T}{\delta}}$ worse than DP-SGD (Theorem 1 in~\cite{AbadiCGMMTZ16}). However, we do not think our algorithm has a worse privacy guarantee inherently. We observe that we are using the standard adaptive composition technique. An interesting open problem is how to better characterize the privacy cost, similarly to the R\'enyi privacy accountant, which we leave to  future work.

\section{Experiments}
\label{sec:exp}
In this section, we conduct experiments for the word embedding algorithm~\cite{MikolovCCD13}, where sparsity inherently exists in the gradients. First, we provide our implementation details, and then we present the performance for our sparse algorithm. We show that our sparse algorithm can achieve better utility at the comparable level of privacy.

\subsection{Model architecture}
The model we consider is the CBOW (Continuous Bag Of Words) version of Word2Vec~\cite{MikolovCCD13}, which is a popular model in the literature. However, training a CBOW model is extremely slow: all the model parameters have to be updated by every batch of the training samples. To accelerate the process and remove the waste of negligible update of all parameters, we modify the optimization objective with the technique of ``Negative Sampling"~\cite{MikolovSCCD13}, where only a small percentage of the parameters are updated in  each training iteration. Specifically, for a pair of target and context words, we randomly pick up a set of negative examples from the vocabulary, which we denote by~$N$. For each sample, which includes a target word $w_t$, a context word $w_c$, and a set of negative words $\{w_{n,i}\colon i \in U\}$, the loss function is defined as follows:
\[
\ell (w_t, w_c, N) = -\log(\sigma(e_t^T e_c)) - \sum_{i \in N} \log (\sigma(-e_t^T e_{n,i} )),
\]
where $e_t$, $e_c$, $e_{n,i}$ denote embeddings of $w_t$, $w_c$, $w_{n,i}$, respectively, and $\sigma$ denotes the sigmoid function.

We run our experiments on the Brown corpus~\cite{FrancisK79}\footnote{Apache-2.0 License.}. In the preprocessing step, we first remove the least frequent and stop words, reducing the vocabulary size to 1{,}000.  The embedding size is set to $100$ for each word. Therefore, the overall number of parameters in the model is $1\textrm{K} \times 100 = 100\textrm{K}$. We choose a window of size $4$, and set $|U|$ to be 8, which means that each sample contains one target word, one context word, and $8$ negative words. Finally, our data contains training, validation, and testing datasets with sizes 200K, 100K, 200K, respectively.

\subsection{Hyperparameter tuning}

Hyperparameter tuning for neural networks requires training several models with various combinations of hyperparameters, which results in privacy cost increase. For simplicity, we just assume our validation dataset is public in this experiment. In other words, no additional privacy cost is incurred by tuning hyperparameters on the validation dataset.

\subsection{Training process}

We implement our models with Opacus~\cite{Opacus}, a library for training differentially private PyTorch models. We set the batch size $b = 20$, and the clipping norm $S_1=15$. We train our models for $20$ epochs, by an Adam optimizer with learning rate $\eta = 0.001$. 
The experiment is run on a Linux server with 6 CPUs and 50GB RAM, and it takes roughly one day to complete. We did not use any GPUs in this experiment.

\subsection{Empirical results}

\begin{figure*}[bt!]
    \centering
    \begin{subfigure}[b]{0.3\textwidth}
        \centering
        \includegraphics[width=\textwidth]{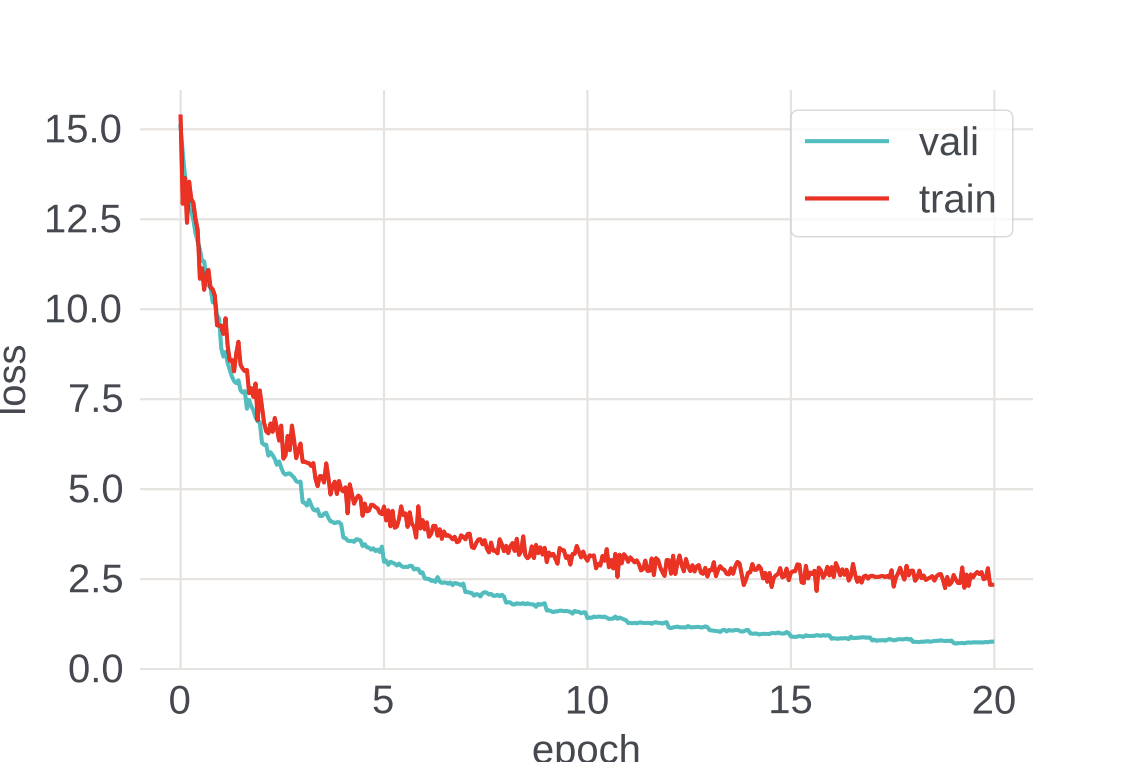}
        \caption{non-private}
        %\label{fig:y equals x}
    \end{subfigure}
    \hfill
    \begin{subfigure}[b]{0.3\textwidth}
        \centering
        \includegraphics[width=\textwidth]{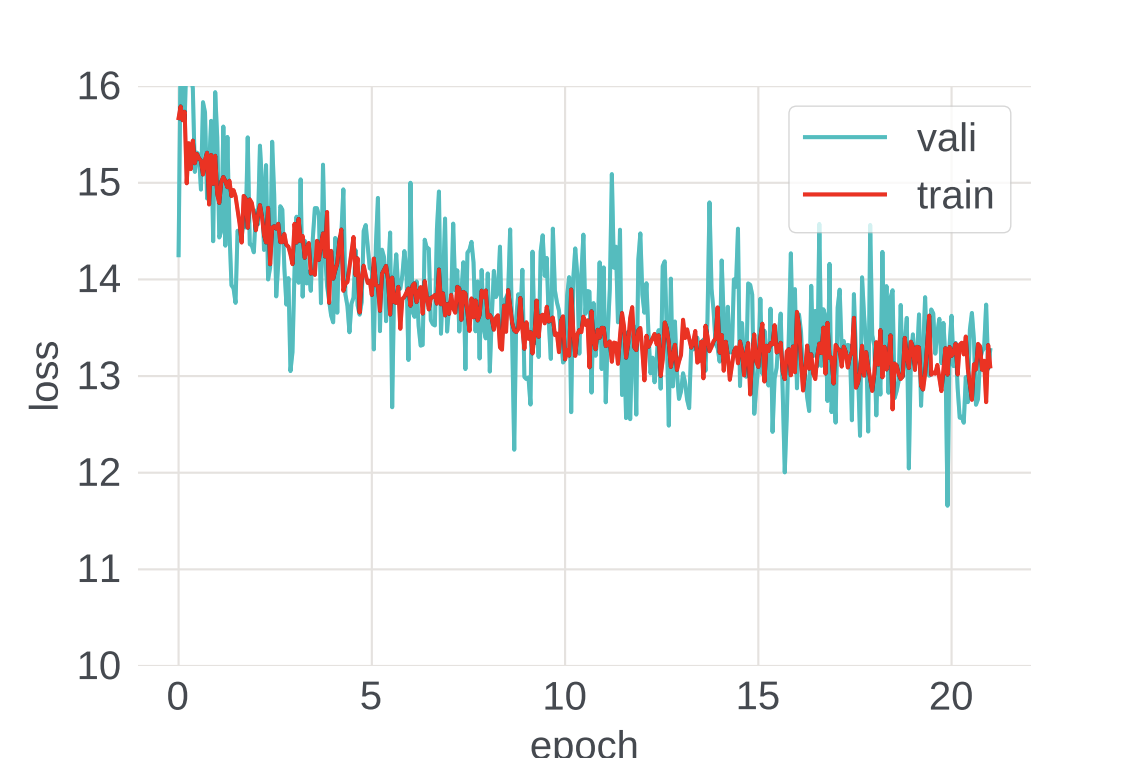}
        \caption{DP-SGD ($\sigma=0.32$)}
        \label{fig:utility_dp_sgd}
    \end{subfigure}
    \hfill
    \begin{subfigure}[b]{0.3\textwidth}
        \centering
        \includegraphics[width=\textwidth]{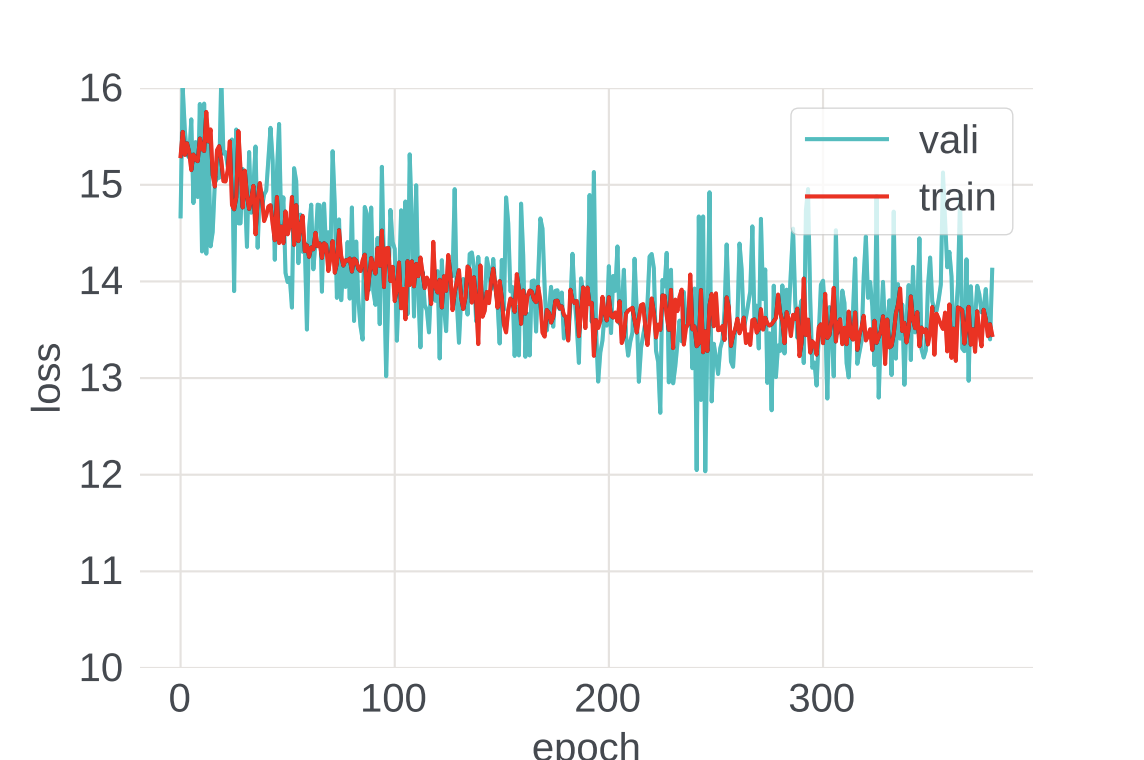}
        \caption{DP-SGD ($\sigma=0.5$)}
        \label{fig:utility_dp_sgd_0.5}
    \end{subfigure}
    \hfill
    \begin{subfigure}[b]{0.3\textwidth}
        \centering
        \includegraphics[width=\textwidth]{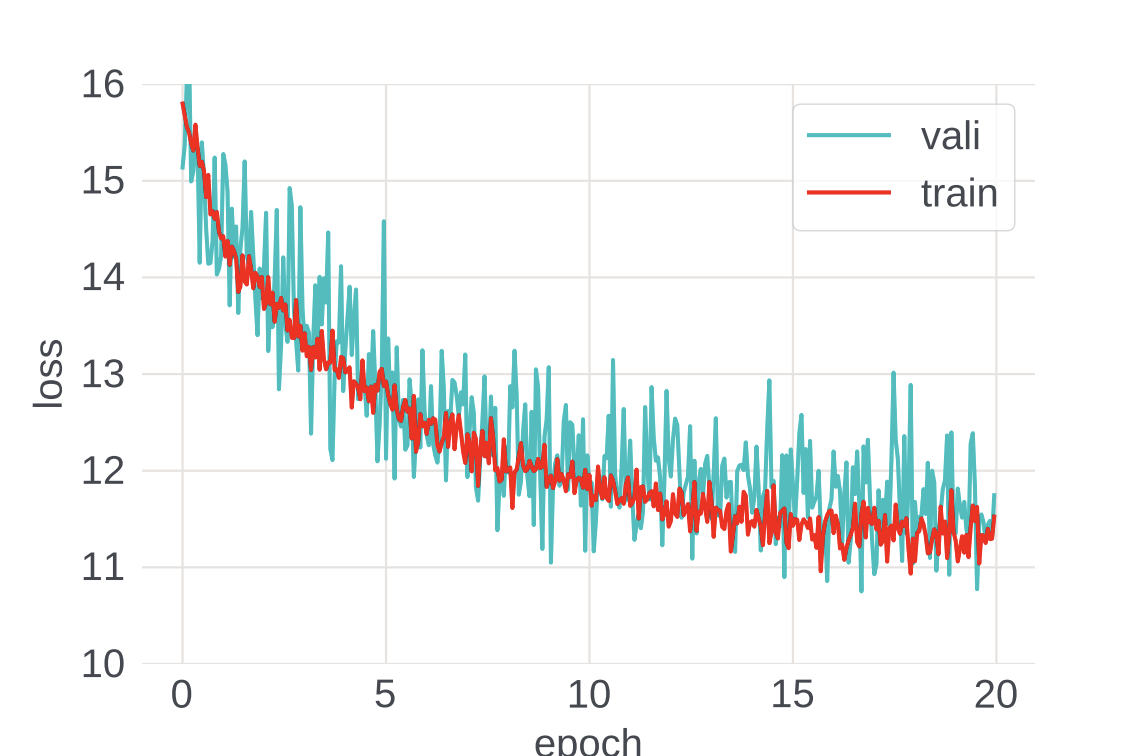}
        \caption{sparse-exponential}
        \label{fig:utility_exponential}
    \end{subfigure}
    \hfill
    \begin{subfigure}[b]{0.3\textwidth}
        \centering
        \includegraphics[width=\textwidth]{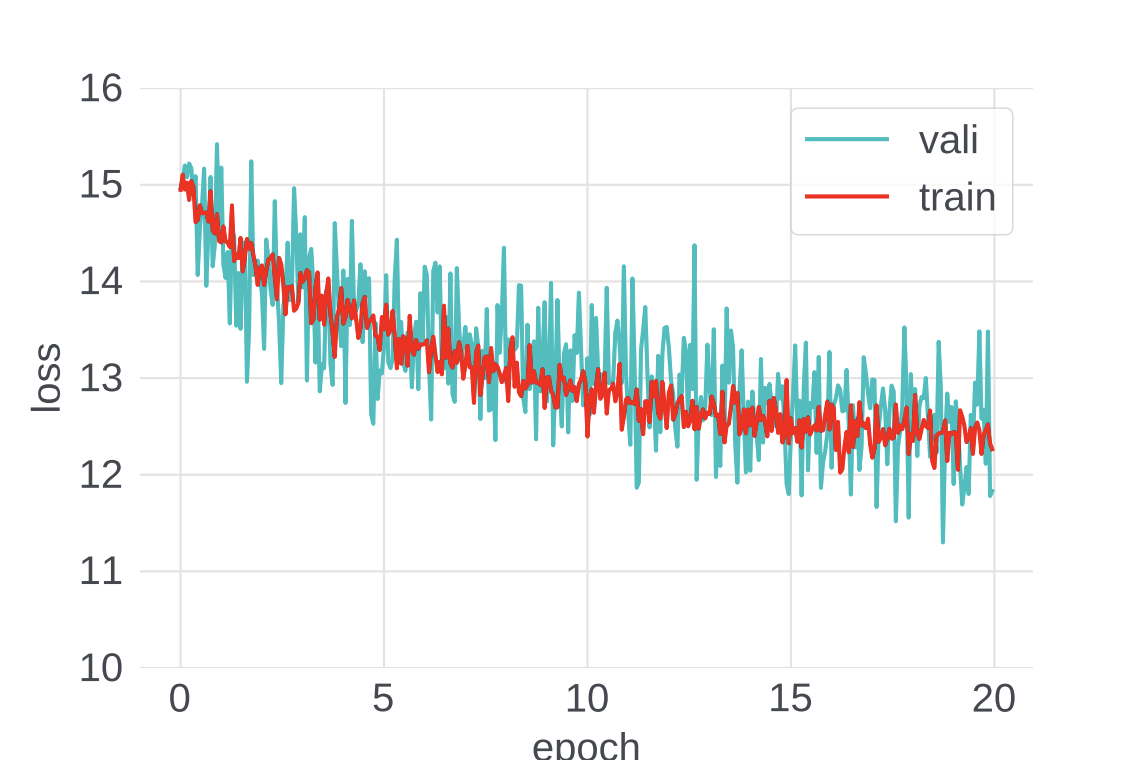}
        \caption{sparse-vector}
        \label{fig:utility_sparse}
    \end{subfigure}
    \hfill
    \begin{subfigure}[b]{0.3\textwidth}
        \centering
        \includegraphics[width=\textwidth]{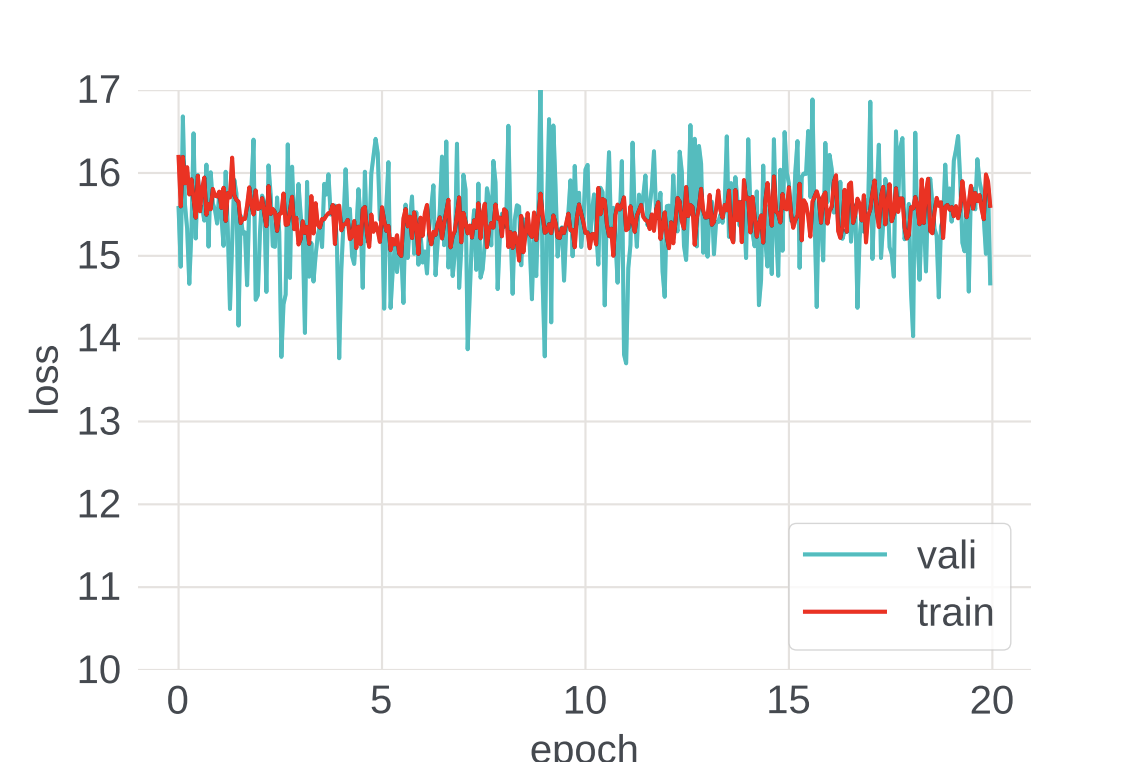}
        \caption{sparse-uniform}
        \label{fig:utility_uniform}
    \end{subfigure}
        \caption{Convergence rate comparison of different algorithms. Our sparse DP algorithms are used in (d), (e), and (f), with coordinated selected by exponential mechanism, sparse vector technique, and random sampling.}
        \label{fig:utility_20}
\end{figure*}

\begin{figure*}[htb!]
    \centering
    \begin{subfigure}[b]{0.2\textwidth}
        \centering
        \includegraphics[width=\textwidth]{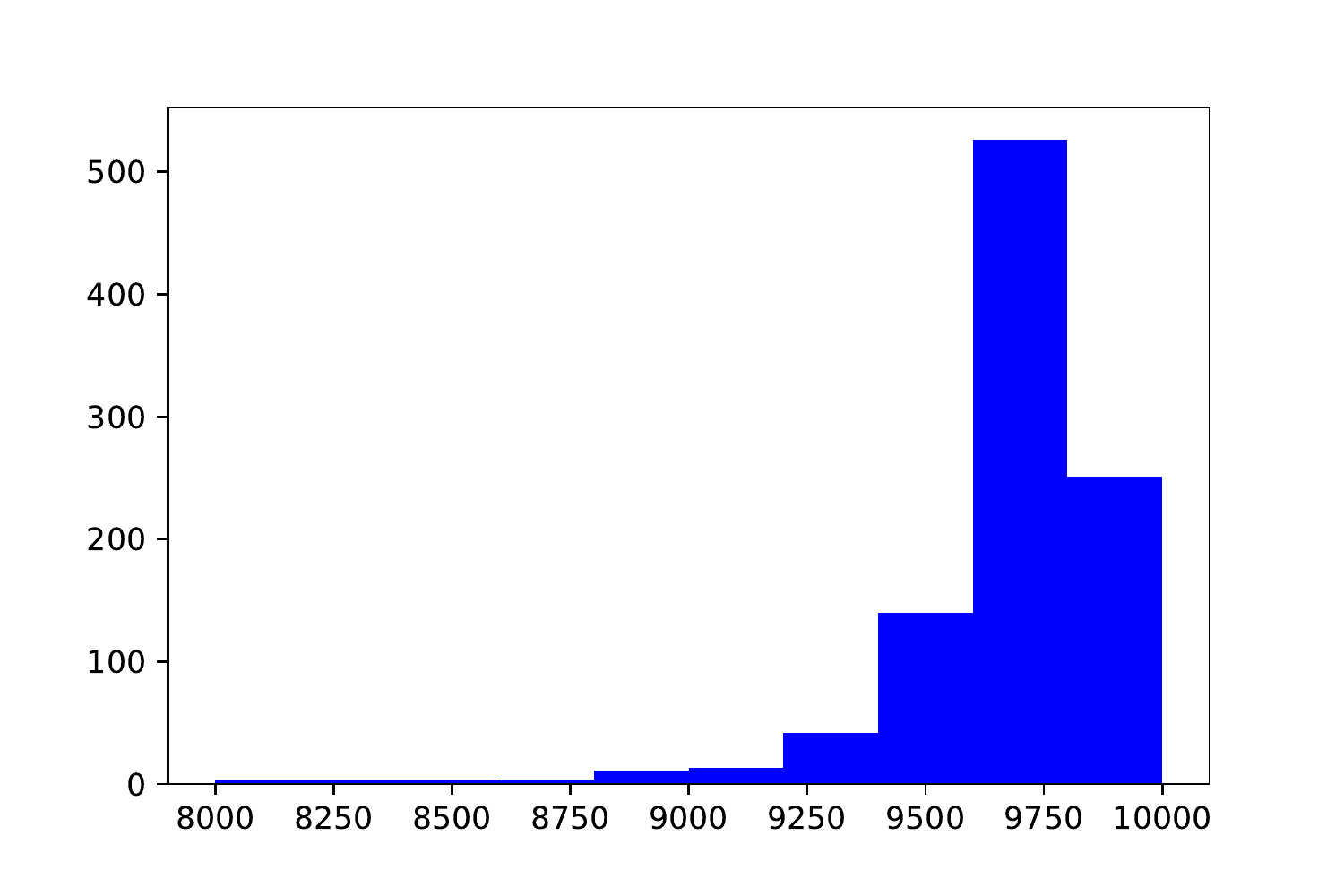}
        %\caption{non-private}
        \label{}
    \end{subfigure}
    \hfill
    \begin{subfigure}[b]{0.2\textwidth}
        \centering
        \includegraphics[width=\textwidth]{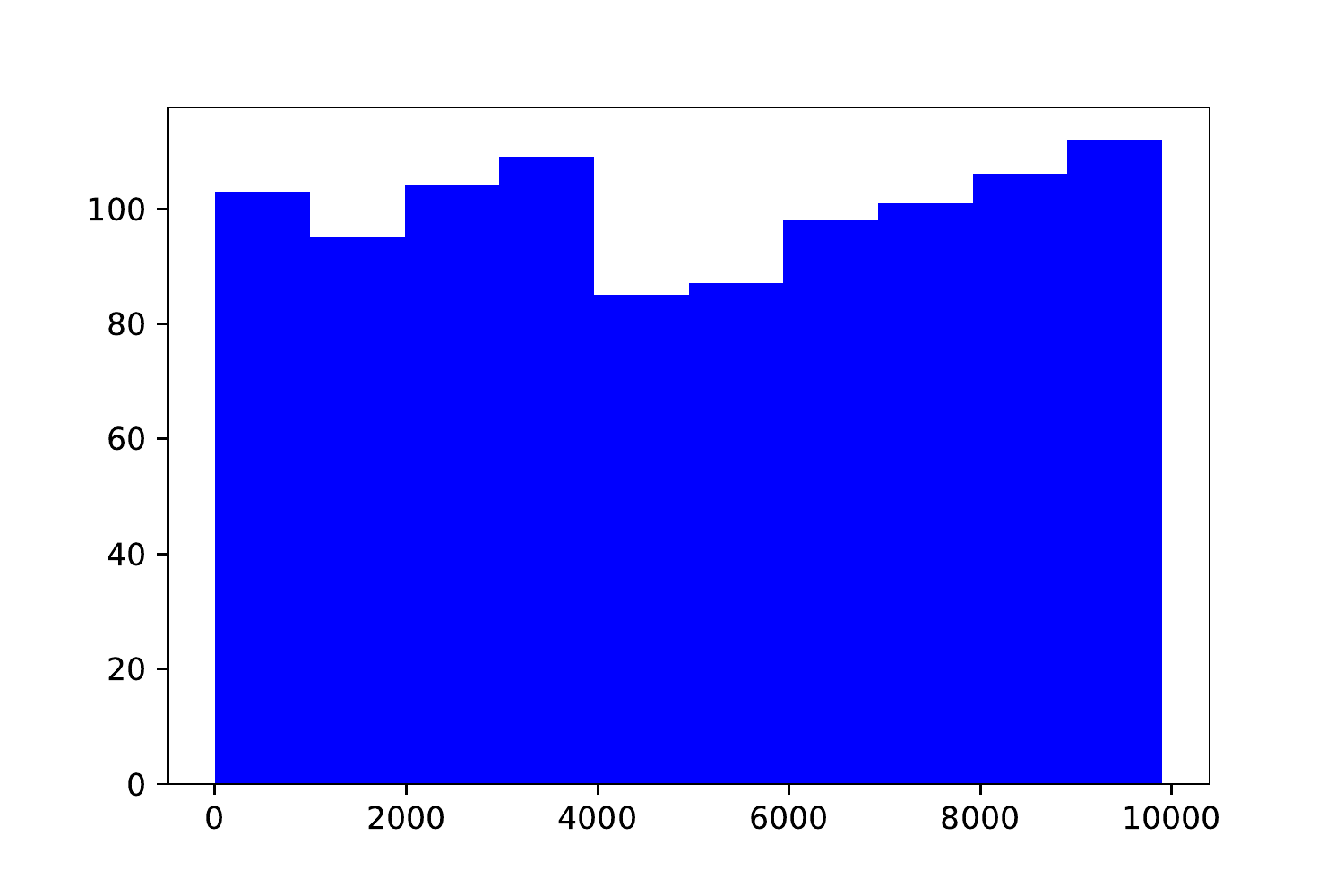}
        %\caption{DP-SGD}
        \label{}
    \end{subfigure}
    \hfill
        \begin{subfigure}[b]{0.2\textwidth}
        \centering
        \includegraphics[width=\textwidth]{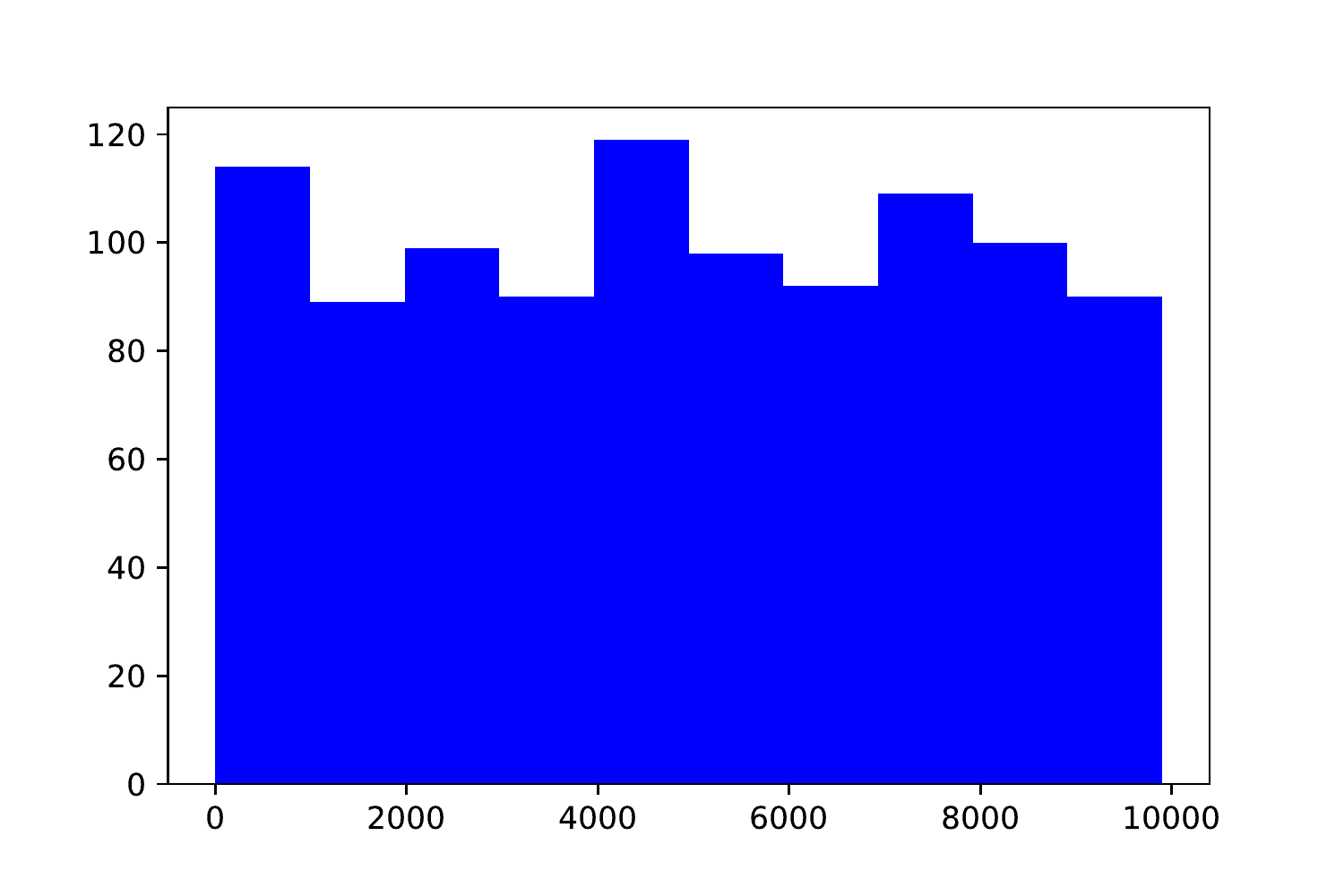}
        %\caption{DP sparse}
        \label{}
    \end{subfigure}
    \hfill
    \begin{subfigure}[b]{0.2\textwidth}
        \centering
        \includegraphics[width=\textwidth]{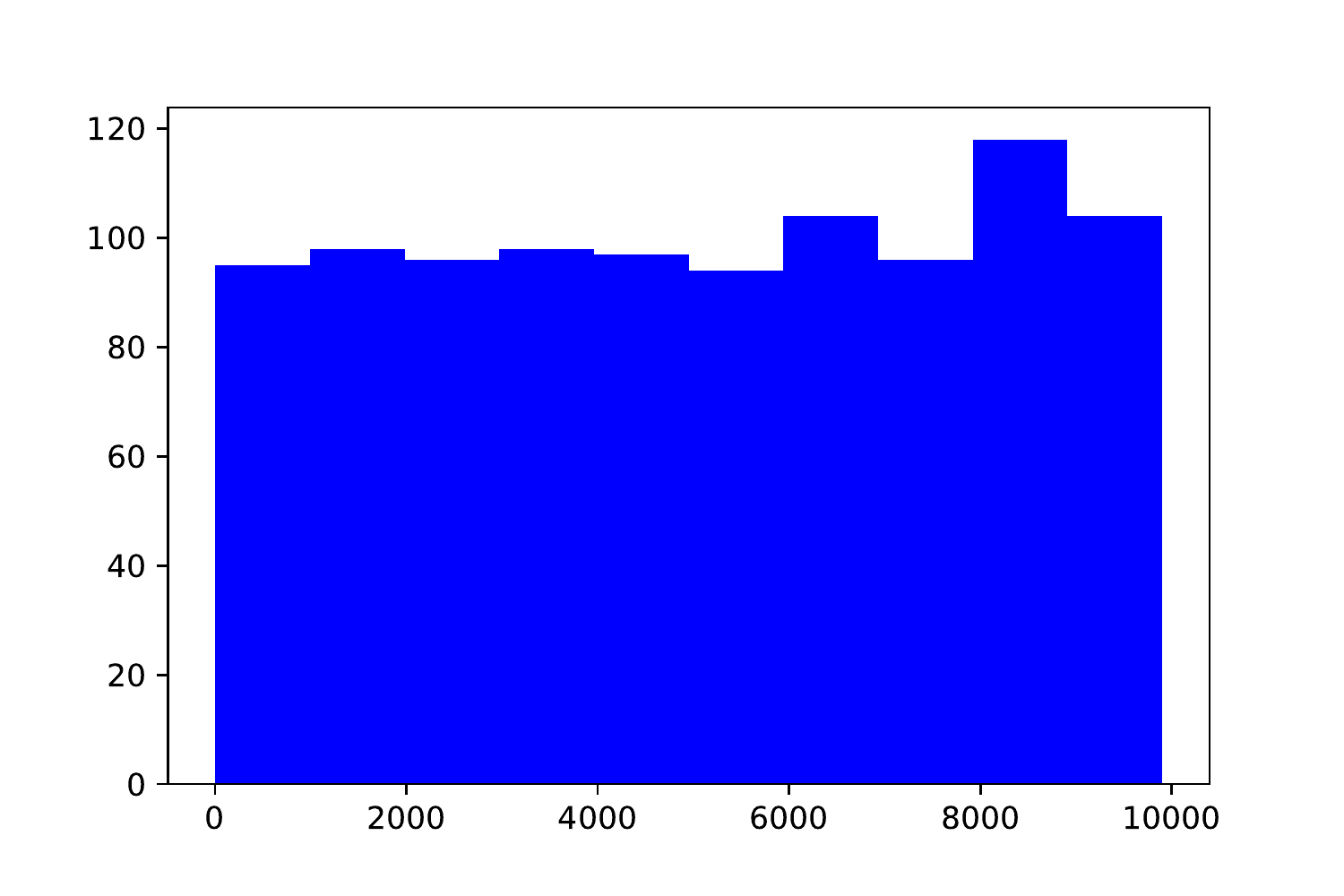}
        %\caption{purely private}
        \label{fig:purely_private_1}
    \end{subfigure}
    \hfill    
    \begin{subfigure}[b]{0.2\textwidth}
        \centering
        \includegraphics[width=\textwidth]{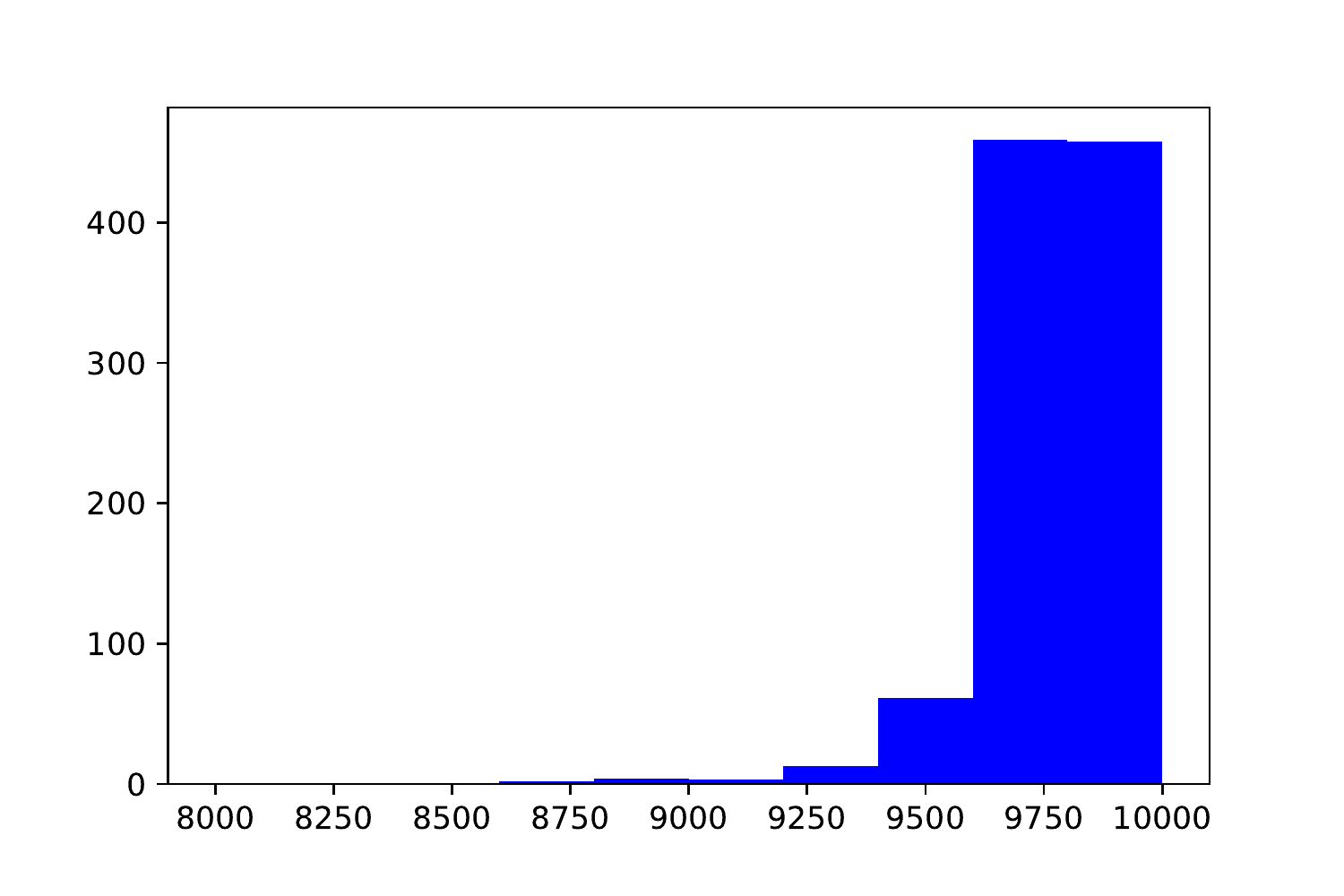}
        %\caption{non-private}
        \label{}
    \end{subfigure}
    \hfill
    \begin{subfigure}[b]{0.2\textwidth}
        \centering
        \includegraphics[width=\textwidth]{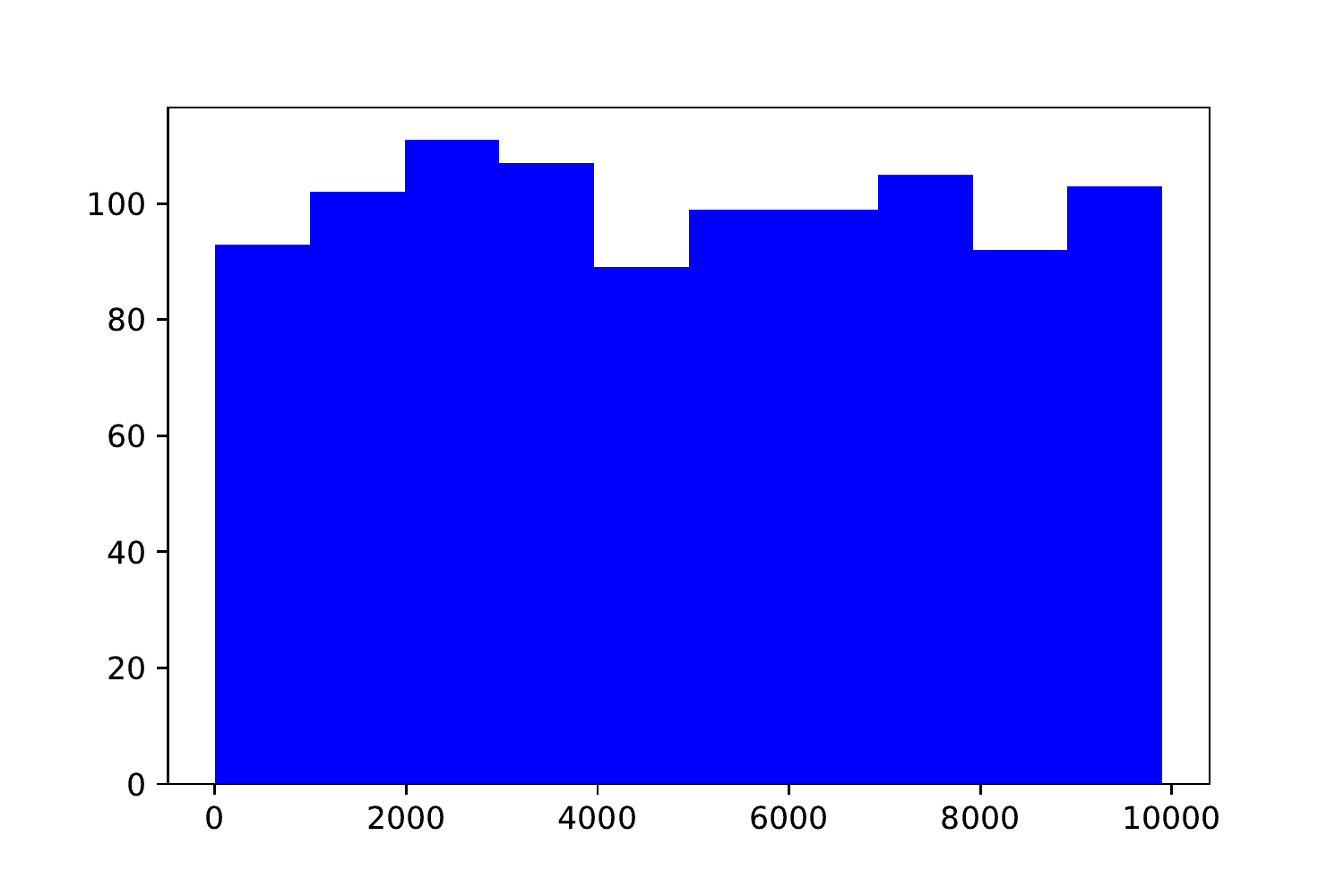}
        %\caption{DP-SGD}
        \label{fig:dpsgd}
    \end{subfigure}
    \hfill
        \begin{subfigure}[b]{0.2\textwidth}
        \centering
        \includegraphics[width=\textwidth]{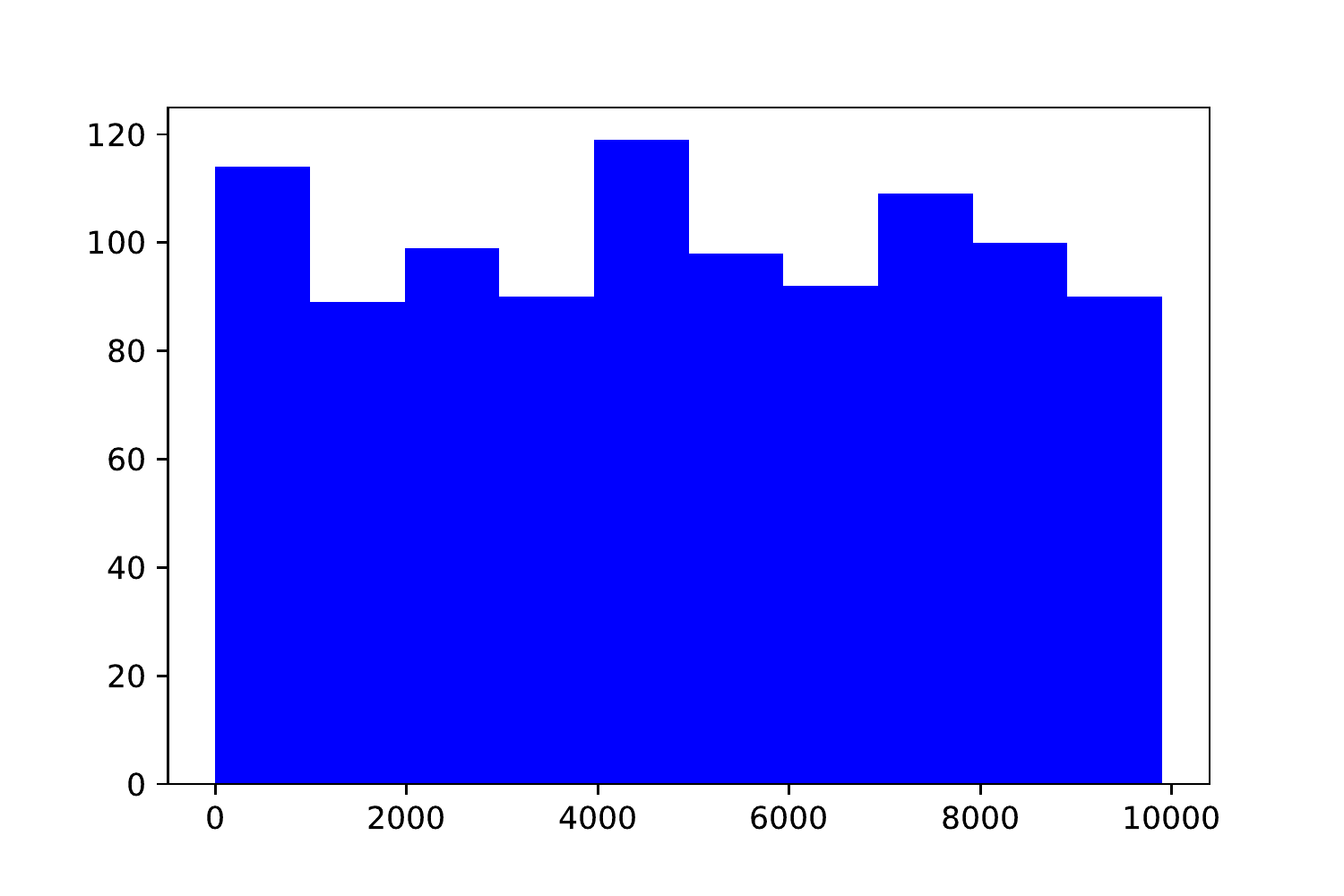}
        %\caption{DP sparse}
        \label{fig:sparse}
    \end{subfigure}
    \hfill
    \begin{subfigure}[b]{0.2\textwidth}
        \centering
        \includegraphics[width=\textwidth]{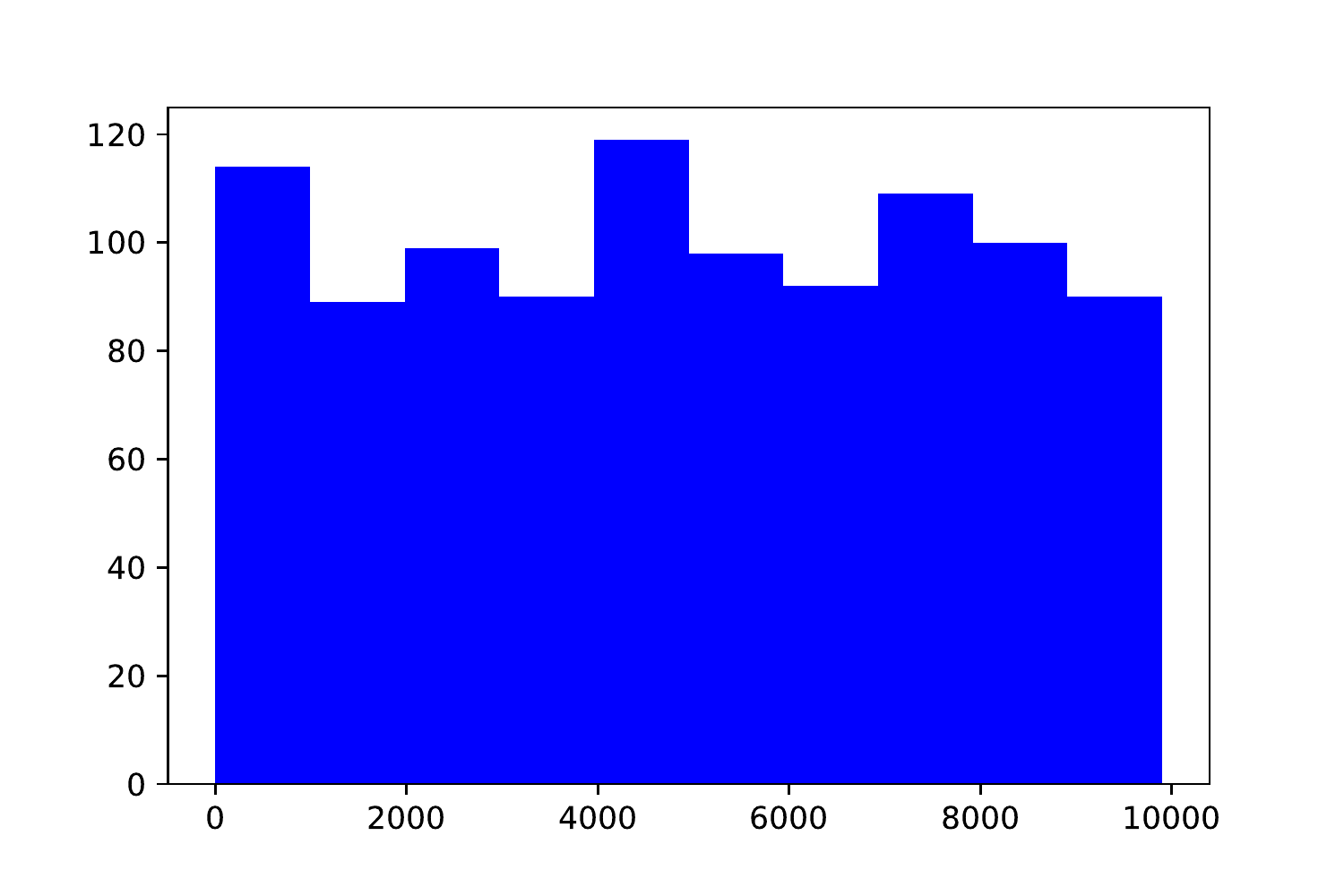}
        %\caption{purely private}
        \label{fig:purely_private}
    \end{subfigure}
    \hfill
        \begin{subfigure}[b]{0.2\textwidth}
        \centering
        \includegraphics[width=\textwidth]{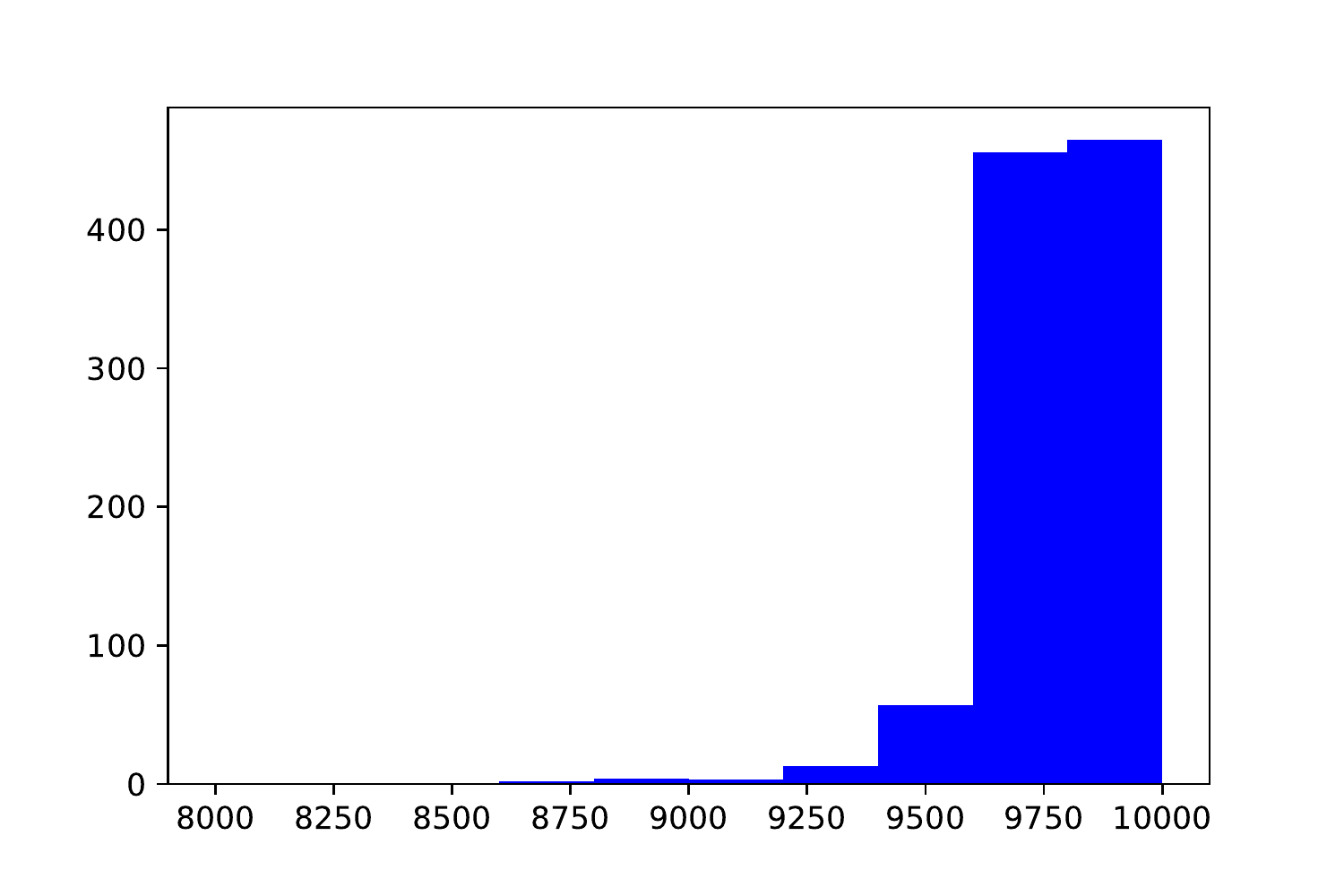}
        \caption{non-private}
        \label{}
    \end{subfigure}
    \hfill
    \begin{subfigure}[b]{0.2\textwidth}
        \centering
        \includegraphics[width=\textwidth]{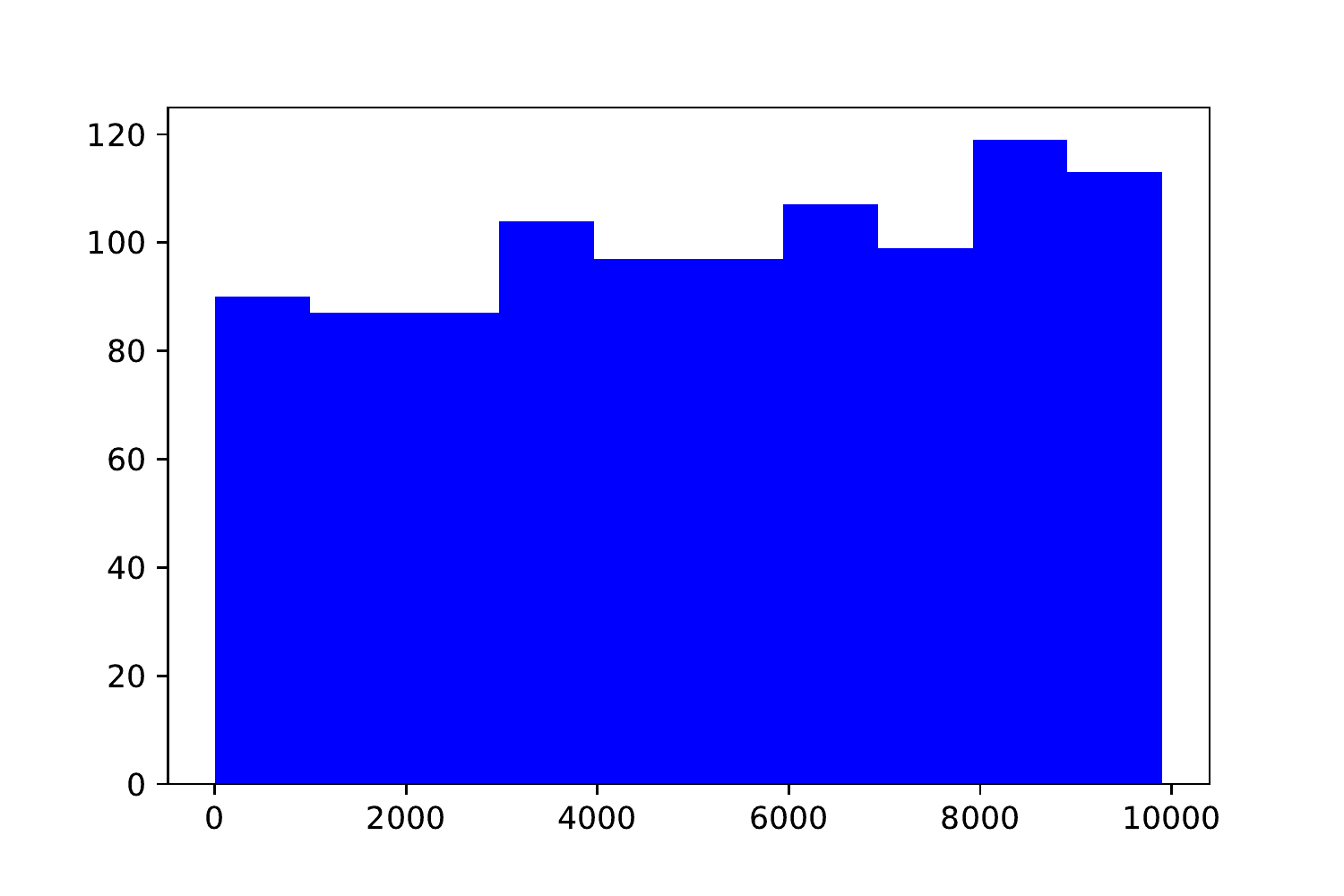}
        \caption{DP-SGD}
        \label{fig:dpsgd}
    \end{subfigure}
    \hfill
        \begin{subfigure}[b]{0.2\textwidth}
        \centering
        \includegraphics[width=\textwidth]{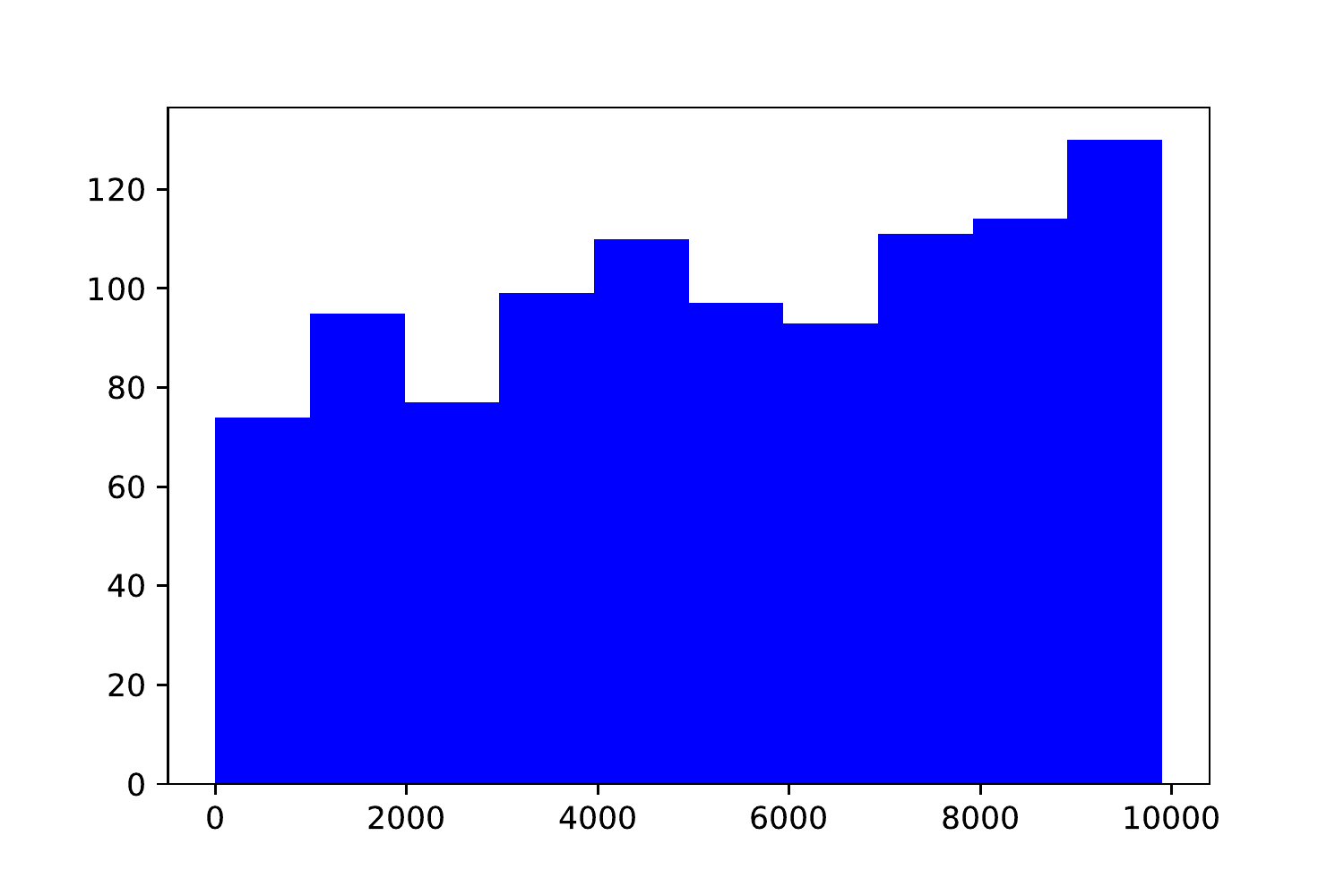}
        \caption{DP Sparse}
        \label{fig:sparse}
    \end{subfigure}
    \hfill
    \begin{subfigure}[b]{0.2\textwidth}
        \centering
        \includegraphics[width=\textwidth]{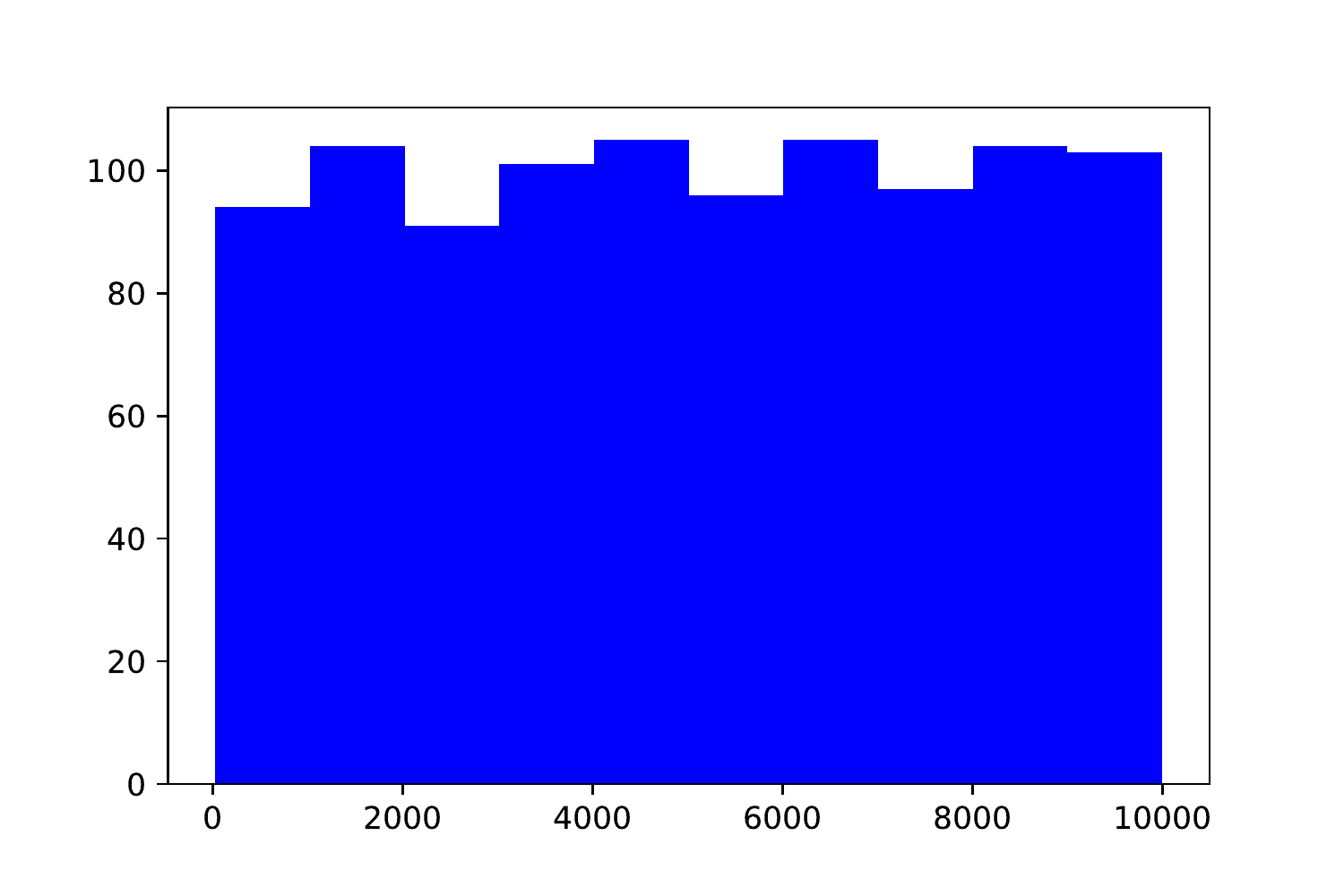}
        \caption{truly private}
        \label{fig:purely_private}
    \end{subfigure}
        % \caption{Canary rank's distribution, $n_c=3$}
        % \label{fig:canary_3}
        \caption{Canary rank's distribution, when $n_c=3$ (top) , $n_c=9$ (middle), and $n_c=15$ (bottom)}
        \label{fig:canary}
\end{figure*}

We compare the performance of our sparse models with DP-SGD, after hyperparameter tuning. With respect to privacy, we fix the same privacy parameters $\eps=30$ and $\delta = 10^{-5}$ for all the algorithms. For the RDP accountant~\cite{RDP-Gaussian}, we choose the noise multiplier $\sigma = 0.32$ for DP-SGD. For the other sparse algorithms, we divide our privacy budget into two parts: $\eps=20$ and $\delta = 5 \times 10^{-6}$ for noise addition, $\eps=10$ and $\delta = 5 \times 10^{-6}$ for private selection. By Theorem~\ref{the:neural}, we set $\sigma = 0.5$ for our sparse algorithms. The hyperparameters for the algorithms are summarized in Table~\ref{table:hyper}.

%\begin{table}[hbt!]
%\centering
%\begin{tabular}{ |c|c|c|c|c| } 
% \hline
%& first gradient clipping norm & second gradient clipping norm & DP Sparse & random \\ 
%&  $S_1$ &  $S_2$ & DP Sparse & random \\ 
%\hline
%DP-SGD  & \textbf{2.62 (.00)} & .007 (.63) & .005 (.85)& .005 (.86)\\ 
%\hline
%sparse-exponential& \textbf{3.24 (.00)} & .004 (.88) & .007 (.68)& .010 (.33) \\ 
%\hline
%sparse-vector & \textbf{3.27 (.00)} & .011 (.31) & \textbf{.026 (.00)} & .002 (.98)  \\ 
% \hline
% sparse-uniform & \textbf{3.27 (.00)} & .011 (.31) & \textbf{.026 (.00)} & .002 (.98) \\
%  \hline
%\end{tabular}
%\caption{Hyperparameters}
%\end{table}

%\begin{table}[hbt!]
%\centering
%\begin{tabular}{ |c|c|c|c|c| } 
% \hline
%&DP-SGD   & sparse-exponential & sparse-vector &  sparse-uniform \\ 
%\hline
%First gradient clipping norm $S_1$& 15 &15 &15& 15\\ 
%\hline
%Second gradient clipping norm $S_2$& N/A & 1 & 1& 1 \\ 
%\hline
%Sparsity parameter $\gamma$ & N/A  & .011 (.31) & \textbf{.026 (.00)} & .002 (.98)  \\ 
% \hline
%DP selection clipping  norm $S_0$ & N/A  & .011 (.31) & \textbf{.026 (.00)} & .002 (.98)  \\ 
% \hline
%%DP selection  $\eps^\prime$& N/A & .011 (.31) & \textbf{.026 (.00)} & .002 (.98) \\
%%  \hline
%%DP selection clipping  $\delta^\prime$& N/A & .011 (.31) & \textbf{.026 (.00)} & .002 (.98) \\
%%  \hline
%\end{tabular}
%\caption{Hyperparameters for each algorithm }
%\end{table}

\begin{table}[hbt!]
\centering
\begin{tabular}{ |c|c|c|c|c| } 
 \hline
&DP-SGD   & DP sparse  \\ 
\hline
Batch size $b$& 20 & 20 \\ 
\hline
Learning rate $\eta$& 0.001 & 0.001 \\ 
\hline
Epoch & 20 & 20 \\ 
\hline
First gradient clipping norm $S_1$& 15 &15 \\ 
\hline
First gradient clipping norm $S_1$& 15 &15 \\ 
\hline
Second gradient clipping norm $S_2$& N/A & 1  \\ 
\hline
Sparsity parameter $\gamma$ & N/A  & 0.001  \\ 
 \hline
DP selection clipping  norm $S_0$ & N/A  & 0.1  \\ 
 \hline
%DP selection  $\eps^\prime$& N/A & .011 (.31) & \textbf{.026 (.00)} & .002 (.98) \\
%  \hline
%DP selection clipping  $\delta^\prime$& N/A & .011 (.31) & \textbf{.026 (.00)} & .002 (.98) \\
%  \hline
\end{tabular}
\caption{The summary of the hyperparameters }
\label{table:hyper}
\end{table}

From Figure~\ref{fig:utility_20}, we observe that our sparse algorithms have provided much better performance than DP-SGD, both in terms of the training error or test error. Furthermore, the private selection by the exponential mechanism slightly outperforms the sparse vector technique, which is consistent with the results in~\cite{LyuS17}.  

One valid complaint is that our algorithms give extremely weak privacy guarantees, since $\eps=30$ is too large to be practical for most applications. We remark that this value is an upper bound on the privacy loss budget.

As mentioned in Section~\ref{sec:nn}, our $\eps$ computation is quite conservative, and we believe that the true $\eps$ value should be much smaller. Here we justify our conjecture by an empirical method. Note that the gap between the training error and test error (generalization error) can serve as a lower bound on the privacy level, as shown in~\cite{dwork-transfer1}, and~\cite{jung-transfer2}.
%\todo{Huanyu ask from Ilya: any other good reference?}. 
From Figure~\ref{fig:utility_20}, we find that the generalization errors of both DP-SGD and our sparse algorithm are small that are consistent with good privacy guarantees. Furthermore, if we improve the noise multiplier of DP-SGD to the same level of our sparse algorithms ($\sigma=0.5$), we observe that the generalization errors are still comparable between our sparse algorithms (Figure~\ref{fig:utility_exponential} and Figure~\ref{fig:utility_sparse}) and DP-SGD (Figure~\ref{fig:utility_dp_sgd_0.5}), indicating they share similar privacy guarantees.
As a benchmark, non-private algorithm has provided the best training error and test error. However, the huge gap between the training and test error indicates the model has provided almost no privacy guarantees. Besides, the periodic behaviour in the training error also indicates the model has severely memorized the training dataset.

\subsection{Evaluation for unintended memorization}
In this section, we use another method to estimate the privacy level of our model. Specifically, we follow the Secret Sharer frameworks proposed from~\cite{CarliniLEKS19}, which aims to measure the unintended memorization of rarely-occurring phrases in the dataset. This method has been further explored in recent works~\cite{JayaramanE19, RamaswamyTMAMB20}.

First, we randomly generate $1{,}000$ canaries, each containing three words. The reason we opt for inserting three-word canaries is that computing the ranks for longer canaries is time-consuming. Each word in a canary is uniformly randomly chosen from the 1K vocabulary. This is because we want to measure \textit{unintended memorization} of our models, i.e., the memorization of atypical phrases in the language model, which is in fact orthogonal to our learning task. Two examples of our canaries are ``mother government opportunity'' and ``prices effort me''.

Next, we insert all these canaries into random positions in our original dataset, each canary appearing exactly $n_c$ times. Then we train our models as before. Note that the canaries have a trivial impact on our models, since the cumulative number of inserted phrases is relatively small relative to the size of the original dataset.

We use the Random Sampling method, as proposed in~\cite{CarliniLEKS19}, to measure whether the canary is memorized by our model. Specifically, for a canary $c = \{c_0, c_1, c_2\}$, we define the log-perplexity of the model $\theta$ on $c$ as $ P_{\theta} (c) = -\log \probof{c_1|c_0} -\log \probof{c_2|c_1, c_0}$. We define the rank of the canary as $\mathit{rank}_{\theta}(c) = \absv{ \{ c_1^\prime \neq c_2^\prime\colon P_{\theta} ( \{c_0, c_1^\prime, c_2^\prime\}) \ge P_{\theta} (c)\}, c_1^\prime, c_2^\prime \in V}$, where $V$ is the vocabulary. Intuitively, a high rank indicates the model highly favors the canary as compared to random chance. In other words, the model has ``memorized'' the canary, suggesting a privacy violation. We note that when computing each canary's rank, it is time-consuming to enumerate all the possible phrases. Therefore, we randomly pick up 10K phrases from the domain and compute $c$'s rank in the subset instead.

In Figure~\ref{fig:canary}, we show the distributions of the rank when $n_c=3$, $n_c=9$ and $n_c=15$. We run the experiments as mentioned above with non-private algorithm, DP-SGD, and Algorithm~\ref{alg:privatenn}  (instantiated with the exponential mechanism for private selection), where we use the same hyperparameters as in the previous experiments. Note that for DP-SGD and DP Sparse, we are using exactly the same noise parameter $\sigma = 0.5$. The experiment is designed to validate the hypothesis that the trained models exhibit similar levels of unintended memorization.

As a benchmark, we randomly select $1{,}000$ phrases that are outside of the training dataset, and we plot the histogram of the rank in Figure~\ref{fig:purely_private}. We find that it is very close to the uniform distribution (confirmed by the chi-squared goodness of fit test). This is indeed expected since the process is equivalent to uniformly randomly drawing one sample from an ordered set, since the phrase is independent of the trained model. 

The top two rows of Figure~\ref{fig:canary} report results for small $n_c$ ($n_c=3$ and 9, respectively). Unlike the non-private training, both DP-SGD and DP Sparse result in histograms close to the uniform distribution (Figures \ref{fig:dpsgd} and~\ref{fig:sparse}). 

We also compute the chi-squared distance, and the $p$-values of Pearson's chi-squared tests (Table~\ref{table:chi-distance}). Where the $p$-values are not statistically significant, the chi-squared test fails to reject the null hypothesis, i.e., that the training procedure preserves privacy.
Visually, the non-private algorithm produces a histogram that is highly concentrated to the right, which indicates that the model has indeed memorized the training dataset. Therefore, we argue that our sparse algorithm gives comparable privacy guarantees as DP-SGD, which are much better than the non-private version.

Finally, this method can be part of hyperparameter tuning, where the chi-squared distance can be an excellent metric to measure privacy leakage.  For the case of $n_c=15$, we observe that Figure~\ref{fig:canary} is consistent with the empirical $\eps$ being quite small (ranging between $0.07$ and $0.11$), following the group property of differential privacy, and the fact that the uniform distribution breaks down at some place between $n_c=9$ and $n_c=15$.

\begin{table}[hbt!]
\centering
\begin{tabular}{ |c|c|c|c|c| } 
 \hline
$n_c$ & non-private & DP-SGD & DP Sparse & random \\ 
\hline
$3$ & \textbf{2.62 (.00)} & .007 (.63) & .005 (.85)& .005 (.86)\\ 
\hline
$9$ & \textbf{3.24 (.00)} & .004 (.88) & .007 (.68)& .010 (.33) \\ 
\hline
$15$ & \textbf{3.27 (.00)} & .011 (.31) & \textbf{.026 (.00)} & .002 (.98)  \\ 
 \hline
\end{tabular}
\caption{The chi-squared goodness of fit with the uniform distribution ($p$-value in parenthesis). Statistically significant results ($p<0.01$) are in bold.}
\label{table:chi-distance}
\end{table}

%\section{Conclusion}
%
%While DP-SGD algorithms and variants have been well studied for solving DP-ERM, the accuracy loss of DP-SGD depends on the model's size $p$. In this paper, we aim at bypassing such dependence by assuming the sparsity of the input dataset. We propose a novel algorithm and provide an empirical study of differentially private word embedding on a real world dataset. The experimental results suggest that our method can achieve greater utility while providing similar privacy guarantee, as compared to the classic DP-SGD algorithm. \todo{Meisam: Huanyu, can you also include your contribution to wide networks and leveraging sparsity in here, emphasizing on that one more time might be a good idea}

 \section{Acknowledgements}

The authors thank Milan Shen and Will Bullock for helpful suggestions and support for this work.

\bibliographystyle{alpha}
\bibliography{biblio}

\newcommand{\etalchar}[1]{$^{#1}$}
\begin{thebibliography}{BJWW{\etalchar{+}}19}

\bibitem[ACG{\etalchar{+}}16]{AbadiCGMMTZ16}
Martin Abadi, Andy Chu, Ian Goodfellow, H~Brendan McMahan, Ilya Mironov, Kunal
  Talwar, and Li~Zhang.
\newblock Deep learning with differential privacy.
\newblock In {\em Proceedings of the 2016 ACM SIGSAC Conference on Computer and
  Communications Security}, pages 308--318, 2016.

\bibitem[AMCDC18]{acs2018differentially}
Gergely Acs, Luca Melis, Claude Castelluccia, and Emiliano De~Cristofaro.
\newblock Differentially private mixture of generative neural networks.
\newblock {\em IEEE Transactions on Knowledge and Data Engineering},
  31(6):1109--1121, 2018.

\bibitem[ASY{\etalchar{+}}18]{AgarwalSYKM18}
Naman Agarwal, Ananda~Theertha Suresh, Felix Xinnan~X Yu, Sanjiv Kumar, and
  Brendan McMahan.
\newblock cpsgd: Communication-efficient and differentially-private distributed
  sgd.
\newblock In {\em Advances in Neural Information Processing Systems}, pages
  7564--7575, 2018.

\bibitem[AZK{\etalchar{+}}18]{abay2018privacy}
Nazmiye~Ceren Abay, Yan Zhou, Murat Kantarcioglu, Bhavani Thuraisingham, and
  Latanya Sweeney.
\newblock Privacy preserving synthetic data release using deep learning.
\newblock In {\em Joint European Conference on Machine Learning and Knowledge
  Discovery in Databases}, pages 510--526. Springer, 2018.

\bibitem[BBG18]{BalleBG18}
Borja Balle, Gilles Barthe, and Marco Gaboardi.
\newblock Privacy amplification by subsampling: Tight analyses via couplings
  and divergences.
\newblock In {\em Advances in Neural Information Processing Systems}, pages
  6277--6287, 2018.

\bibitem[BJWW{\etalchar{+}}19]{beaulieu2019privacy}
Brett~K Beaulieu-Jones, Zhiwei~Steven Wu, Chris Williams, Ran Lee, Sanjeev~P
  Bhavnani, James~Brian Byrd, and Casey~S Greene.
\newblock Privacy-preserving generative deep neural networks support clinical
  data sharing.
\newblock {\em Circulation: Cardiovascular Quality and Outcomes},
  12(7):e005122, 2019.

\bibitem[BPS19]{BagdasaryanPS19}
Eugene Bagdasaryan, Omid Poursaeed, and Vitaly Shmatikov.
\newblock Differential privacy has disparate impact on model accuracy.
\newblock In {\em Advances in Neural Information Processing Systems}, pages
  15479--15488, 2019.

\bibitem[BST14]{BassilyST14}
Raef Bassily, Adam Smith, and Abhradeep Thakurta.
\newblock Private empirical risk minimization: Efficient algorithms and tight
  error bounds.
\newblock In {\em Proceedings of the 55th Annual IEEE Symposium on Foundations
  of Computer Science}, FOCS '14, pages 464--473, Washington, DC, USA, 2014.
  IEEE Computer Society.

\bibitem[CH11]{ChaudhuriH11}
Kamalika Chaudhuri and Daniel Hsu.
\newblock Sample complexity bounds for differentially private learning.
\newblock In {\em Proceedings of the 24th Annual Conference on Learning
  Theory}, COLT '11, pages 155--186, 2011.

\bibitem[CLE{\etalchar{+}}19]{CarliniLEKS19}
Nicholas Carlini, Chang Liu, {\'U}lfar Erlingsson, Jernej Kos, and Dawn Song.
\newblock The secret sharer: Evaluating and testing unintended memorization in
  neural networks.
\newblock In {\em 28th USENIX Security Symposium}, pages 267--284, 2019.

\bibitem[CTW{\etalchar{+}}20]{carlini2020extracting}
Nicholas Carlini, Florian Tramer, Eric Wallace, Matthew Jagielski, Ariel
  Herbert-Voss, Katherine Lee, Adam Roberts, Tom Brown, Dawn Song, Ulfar
  Erlingsson, Alina Oprea, and Colin Raffel.
\newblock Extracting training data from large language models, 2020.

\bibitem[CWZ19]{CaiWZ19}
T.~Tony Cai, Yichen Wang, and Linjun Zhang.
\newblock The cost of privacy: Optimal rates of convergence for parameter
  estimation with differential privacy.
\newblock {\em arXiv preprint arXiv:1902.04495}, 2019.

\bibitem[CXX{\etalchar{+}}18]{chen2018differentially}
Qingrong Chen, Chong Xiang, Minhui Xue, Bo~Li, Nikita Borisov, Dali Kaarfar,
  and Haojin Zhu.
\newblock Differentially private data generative models.
\newblock {\em arXiv preprint arXiv:1812.02274}, 2018.

\bibitem[DFH{\etalchar{+}}15]{dwork-transfer1}
Cynthia Dwork, Vitaly Feldman, Moritz Hardt, Toniann Pitassi, Omer Reingold,
  and Aaron~Leon Roth.
\newblock Preserving statistical validity in adaptive data analysis.
\newblock In {\em Proceedings of the Forty-Seventh Annual ACM Symposium on
  Theory of Computing (STOC)}, page 117–126, 2015.

\bibitem[{Dif}17]{AppleDP17}
{Differential Privacy Team, Apple}.
\newblock Learning with privacy at scale.
\newblock
  \url{https://machinelearning.apple.com/docs/learning-with-privacy-at-scale/appledifferentialprivacysystem.pdf},
  December 2017.

\bibitem[DKY17]{DingKY17}
Bolin Ding, Janardhan Kulkarni, and Sergey Yekhanin.
\newblock Collecting telemetry data privately.
\newblock In {\em Advances in Neural Information Processing Systems 30}, NIPS
  '17, pages 3571--3580. Curran Associates, Inc., 2017.

\bibitem[DMNS06]{DworkMNS06}
Cynthia Dwork, Frank McSherry, Kobbi Nissim, and Adam Smith.
\newblock Calibrating noise to sensitivity in private data analysis.
\newblock In {\em Proceedings of the 3rd Conference on Theory of Cryptography},
  TCC '06, pages 265--284, Berlin, Heidelberg, 2006. Springer.

\bibitem[DNR{\etalchar{+}}09]{DworkNRRV09}
Cynthia Dwork, Moni Naor, Omer Reingold, Guy~N. Rothblum, and Salil Vadhan.
\newblock On the complexity of differentially private data release: Efficient
  algorithms and hardness results.
\newblock In {\em Proceedings of the 41st Annual ACM Symposium on the Theory of
  Computing}, STOC '09, pages 381--390, New York, NY, USA, 2009. ACM.

\bibitem[DR14]{DworkR14}
Cynthia Dwork and Aaron Roth.
\newblock The algorithmic foundations of differential privacy.
\newblock {\em Foundations and Trends in Machine Learning}, 9(3--4):211--407,
  2014.

\bibitem[DR19a]{durfee2019practical}
David Durfee and Ryan~M Rogers.
\newblock Practical differentially private top-k selection with
  pay-what-you-get composition.
\newblock {\em Advances in Neural Information Processing Systems},
  32:3532--3542, 2019.

\bibitem[DR19b]{DurfeeR19}
David Durfee and Ryan~M. Rogers.
\newblock Practical differentially private top-$k$ selection with
  pay-what-you-get composition.
\newblock In {\em Advances in Neural Information Processing Systems}, pages
  3532--3542, 2019.

\bibitem[EPK14]{ErlingssonPK14}
{\'U}lfar Erlingsson, Vasyl Pihur, and Aleksandra Korolova.
\newblock {RAPPOR}: Randomized aggregatable privacy-preserving ordinal
  response.
\newblock In {\em Proceedings of the 2014 ACM Conference on Computer and
  Communications Security}, CCS '14, pages 1054--1067, New York, NY, USA, 2014.
  ACM.

\bibitem[FJR15]{fredrikson2015model}
Matt Fredrikson, Somesh Jha, and Thomas Ristenpart.
\newblock Model inversion attacks that exploit confidence information and basic
  countermeasures.
\newblock In {\em Proceedings of the 22nd ACM SIGSAC Conference on Computer and
  Communications Security}, pages 1322--1333, 2015.

\bibitem[FK79]{FrancisK79}
W~Nelson Francis and Henry Kucera.
\newblock Brown corpus manual.
\newblock {\em Letters to the Editor}, 5(2):7, 1979.

\bibitem[GL13]{ghadimi2013stochastic}
Saeed Ghadimi and Guanghui Lan.
\newblock Stochastic first-and zeroth-order methods for nonconvex stochastic
  programming.
\newblock {\em SIAM Journal on Optimization}, 23(4):2341--2368, 2013.

\bibitem[HMNZ19]{hejazinia2019deep}
Meisam Hejazinia, Pavlos Mitsoulis-Ntompos, and Serena Zhang.
\newblock Deep personalized re-targeting.
\newblock In {\em 2019 IEEE/ACM International Conference on Advances in Social
  Networks Analysis and Mining (ASONAM)}, pages 1148--1154. IEEE, 2019.

\bibitem[HSLH17]{hu2017mitigating}
Yan Hu, Weisong Shi, Hong Li, and Xiaohui Hu.
\newblock Mitigating data sparsity using similarity reinforcement-enhanced
  collaborative filtering.
\newblock {\em ACM Transactions on Internet Technology (TOIT)}, 17(3):1--20,
  2017.

\bibitem[HT10]{HardtT10}
Moritz Hardt and Kunal Talwar.
\newblock On the geometry of differential privacy.
\newblock In {\em Proceedings of the 42nd Annual ACM Symposium on the Theory of
  Computing}, STOC '10, pages 705--714, New York, NY, USA, 2010. ACM.

\bibitem[JE19]{JayaramanE19}
Bargav Jayaraman and David Evans.
\newblock Evaluating differentially private machine learning in practice.
\newblock In {\em 28th USENIX Security Symposium}, pages 1895--1912, 2019.

\bibitem[JLN{\etalchar{+}}20]{jung-transfer2}
Christopher Jung, Katrina Ligett, Seth Neel, Aaron Roth, Saeed
  Sharifi{-}Malvajerdi, and Moshe Shenfeld.
\newblock A new analysis of differential privacy's generalization guarantees.
\newblock In Thomas Vidick, editor, {\em 11th Innovations in Theoretical
  Computer Science Conference, {ITCS} 2020, January 12-14, 2020, Seattle,
  Washington, {USA}}, volume 151 of {\em LIPIcs}, pages 31:1--31:17, 2020.

\bibitem[JT14]{JainT14}
Prateek Jain and Abhradeep~Guha Thakurta.
\newblock {(Near)} dimension independent risk bounds for differentially private
  learning.
\newblock In {\em International Conference on Machine Learning}, pages
  476--484, 2014.

\bibitem[KJ16]{kasiviswanathanJ16}
Shiva~Prasad Kasiviswanathan and Hongxia Jin.
\newblock Efficient private empirical risk minimization for high-dimensional
  learning.
\newblock In {\em International Conference on Machine Learning}, pages
  488--497, 2016.

\bibitem[KOV15]{KairouzOV15}
Peter Kairouz, Sewoong Oh, and Pramod Viswanath.
\newblock The composition theorem for differential privacy.
\newblock In {\em International conference on machine learning}, pages
  1376--1385, 2015.

\bibitem[LK20]{lee2020differentially}
Jaewoo Lee and Daniel Kifer.
\newblock Differentially private deep learning with direct feedback alignment.
\newblock {\em arXiv preprint arXiv:2010.03701}, 2020.

\bibitem[LSL17]{LyuS17}
Min Lyu, Dong Su, and Ninghui Li.
\newblock Understanding the sparse vector technique for differential privacy.
\newblock {\em Proceedings of the VLDB Endowment}, 10(6), 2017.

\bibitem[MAM{\etalchar{+}}18]{mcmahan2018general}
Brendan McMahan, Galen Andrew, Ilya Mironov, Nicolas Papernot, Peter Kairouz,
  Steve Chien, and Úlfar Erlingsson.
\newblock A general approach to adding differential privacy to iterative
  training procedures.
\newblock 2018.
\newblock Workshop on Privacy Preserving Machine Learning (NeurIPS 2018).

\bibitem[MCCD13]{MikolovCCD13}
Tomas Mikolov, Kai Chen, Greg Corrado, and Jeffrey Dean.
\newblock Efficient estimation of word representations in vector space.
\newblock {\em arXiv preprint arXiv:1301.3781}, 2013.

\bibitem[MNHZB19]{mitsoulis2019simple}
Pavlos Mitsoulis-Ntompos, Meisam Hejazinia, Serena Zhang, and Travis Brady.
\newblock A simple deep personalized recommendation system.
\newblock {\em arXiv preprint arXiv:1906.11336}, 2019.

\bibitem[MRTZ18]{mcmahan2017learning}
Brendan McMahan, Daniel Ramage, Kunal Talwar, and Li~Zhang.
\newblock Learning differentially private recurrent language models.
\newblock In {\em International Conference on Learning Representations (ICLR)},
  2018.

\bibitem[MSC{\etalchar{+}}13]{MikolovSCCD13}
Tomas Mikolov, Ilya Sutskever, Kai Chen, Greg~S Corrado, and Jeff Dean.
\newblock Distributed representations of words and phrases and their
  compositionality.
\newblock {\em Advances in neural information processing systems},
  26:3111--3119, 2013.

\bibitem[MT07]{McSherryT07}
Frank McSherry and Kunal Talwar.
\newblock Mechanism design via differential privacy.
\newblock In {\em Proceedings of the 48th Annual IEEE Symposium on Foundations
  of Computer Science}, FOCS '07, pages 94--103, Washington, DC, USA, 2007.
  IEEE Computer Society.

\bibitem[MTZ19]{RDP-Gaussian}
Ilya Mironov, Kunal Talwar, and Li~Zhang.
\newblock R$\backslash$'enyi differential privacy of the sampled gaussian
  mechanism.
\newblock {\em arXiv preprint arXiv:1908.10530}, 2019.

\bibitem[Opa20]{Opacus}
Introducing opacus: A high-speed library for training pytorch models with
  differential privacy.
\newblock \url{https://github.com/pytorch/opacus}, August 2020.

\bibitem[PKP{\etalchar{+}}18]{popov2018distributed}
Vadim Popov, Mikhail Kudinov, Irina Piontkovskaya, Petr Vytovtov, and Alex
  Nevidomsky.
\newblock Distributed fine-tuning of language models on private data.
\newblock In {\em International Conference on Learning Representations}, 2018.

\bibitem[PSM{\etalchar{+}}18]{papernot2018scalable}
Nicolas Papernot, Shuang Song, Ilya Mironov, Ananth Raghunathan, Kunal Talwar,
  and Ulfar Erlingsson.
\newblock Scalable private learning with {PATE}.
\newblock In {\em International Conference on Learning Representations}, 2018.

\bibitem[RTM{\etalchar{+}}20]{RamaswamyTMAMB20}
Swaroop Ramaswamy, Om~Thakkar, Rajiv Mathews, Galen Andrew, H~Brendan McMahan,
  and Fran{\c{c}}oise Beaufays.
\newblock Training production language models without memorizing user data.
\newblock {\em arXiv preprint arXiv:2009.10031}, 2020.

\bibitem[SS15]{ShokriS15}
Reza Shokri and Vitaly Shmatikov.
\newblock Privacy-preserving deep learning.
\newblock In {\em Proceedings of the 22nd ACM SIGSAC conference on computer and
  communications security}, pages 1310--1321, 2015.

\bibitem[SU17]{SteinkeU17b}
Thomas Steinke and Jonathan Ullman.
\newblock Tight lower bounds for differentially private selection.
\newblock In {\em Proceedings of the 58th Annual IEEE Symposium on Foundations
  of Computer Science}, FOCS '17, pages 552--563, Washington, DC, USA, 2017.
  IEEE Computer Society.

\bibitem[TAM19]{thakkar2019differentially}
Om~Thakkar, Galen Andrew, and H~Brendan McMahan.
\newblock Differentially private learning with adaptive clipping.
\newblock {\em arXiv preprint arXiv:1905.03871}, 2019.

\bibitem[TTZ15]{TalwarTZ15}
Kunal Talwar, Abhradeep Thakurta, and Li~Zhang.
\newblock Nearly-optimal private {LASSO}.
\newblock In {\em Advances in Neural Information Processing Systems 28}, NIPS
  '15, pages 3025--3033. Curran Associates, Inc., 2015.

\bibitem[VTJ19]{vu2019dpugc}
XS~Vu, SN~Tran, and L~Jiang.
\newblock dpugc: Learn differentially private representation for user generated
  contents.
\newblock In {\em Proceedings of the 20th International Conference on
  Computational Linguistics and Intelligent Text Processing}, pages 1--16,
  2019.

\bibitem[WBK19]{wang2019subsampled}
Yu-Xiang Wang, Borja Balle, and Shiva~Prasad Kasiviswanathan.
\newblock Subsampled {R\'e}nyi differential privacy and analytical moments
  accountant.
\newblock In {\em The 22nd International Conference on Artificial Intelligence
  and Statistics}, pages 1226--1235. PMLR, 2019.

\bibitem[WCX19]{wang2019differentially}
Di~Wang, Changyou Chen, and Jinhui Xu.
\newblock Differentially private empirical risk minimization with non-convex
  loss functions.
\newblock In {\em International Conference on Machine Learning}, pages
  6526--6535. PMLR, 2019.

\bibitem[WJEG19]{WangJEG19}
Lingxiao Wang, Bargav Jayaraman, David Evans, and Quanquan Gu.
\newblock Efficient privacy-preserving nonconvex optimization.
\newblock {\em arXiv preprint arXiv:1910.13659}, 2019.

\bibitem[WYX17]{wang2018differentially}
Di~Wang, Minwei Ye, and Jinhui Xu.
\newblock Differentially private empirical risk minimization revisited: Faster
  and more general.
\newblock In {\em Advances in Neural Information Processing Systems},
  volume~30, pages 2722--2731. Curran Associates, Inc., 2017.

\bibitem[ZWB20]{zhou2020bypassing}
Yingxue Zhou, Zhiwei~Steven Wu, and Arindam Banerjee.
\newblock Bypassing the ambient dimension: Private {SGD} with gradient subspace
  identification.
\newblock {\em arXiv preprint arXiv:2007.03813}, 2020.

\bibitem[ZZMW17]{zhang2017efficient}
Jiaqi Zhang, Kai Zheng, Wenlong Mou, and Liwei Wang.
\newblock Efficient private erm for smooth objectives.
\newblock In {\em IJCAI}, 2017.

\end{thebibliography}

% \appendix
% \section{More Details for the Experiments}
% \begin{figure}[hbt!]
%     \centering
%     \begin{subfigure}[b]{0.4\textwidth}
%         \centering
%         \includegraphics[width=\textwidth]{plot/batch200_non_private.png}
%         \caption{non-private}
%         %\label{}
%     \end{subfigure}
%     \hfill
%     \begin{subfigure}[b]{0.4\textwidth}
%         \centering
%         \includegraphics[width=\textwidth]{plot/batch200_dp_sgd.png}
%         \caption{DP-SGD ($\sigma=0.5$)}
%         \label{}
%     \end{subfigure}
%     \hfill
%     \begin{subfigure}[b]{0.4\textwidth}
%         \centering
%         \includegraphics[width=\textwidth]{plot/batch200_exp.png}
%         \caption{sparse-exponential)}
%         \label{}
%     \end{subfigure}
%     \hfill
%     \begin{subfigure}[b]{0.4\textwidth}
%         \centering
%         \includegraphics[width=\textwidth]{plot/batch200_sparse.png}
%         \caption{sparse-vector}
%         \label{}
%     \end{subfigure}
%         \caption{Convergence rate of the model}
%         \label{}
% \end{figure}

% \begin{table}[ht]
% \centering
% \begin{tabular}{ |c|c|c|c|c| } 
%  \hline
% & non-private & DP-SGD & DP sparse & purely private \\ 
% \hline
% $n_c=3$  & 0  &0.6267  & 0.8547 & 0.8629 \\ 
% \hline
% $n_c=9$  & 0  & 0.8801  & 0.6786 & 0.3283  \\ 
% \hline
% $n_c=15$  & 0  & 0.3100  & 0.0020 & 0.9849   \\ 
%  \hline
% \end{tabular}
% \caption{The $p$-value of the chi-squared test}
% \label{table:chi-distance-p-value}
% \end{table}
\appendix
\section{Proof of Lemma~\ref{lem:theer}}
\label{app:proof}

The first half comes from~\cite{AgarwalSYKM18}. Therefore, it is enough to prove the second half, where we use a similar proof technique with~\cite{AgarwalSYKM18} and~\cite{ghadimi2013stochastic}. 

First observe that, for any $t=1,\ldots, T$, 
\begin{align*}
&~~~\norm{w_{t+1}-w^*}^2 = \norm{w_t - \eta \Delta_t -w^* }^2\\
&= \norm{w_t  -w^* }^2 -2\eta \langle \Delta_t, w_t-w^* \rangle+ \eta^2 \norm{ \Delta_t}^2\\
& = \norm{w_t  -w^* }^2 -2\eta \langle \nabla L(w_t;D )+A_t, w_t-w^* \rangle+ \eta^2\Paren{ \norm{\nabla L(w_t;D )}^2 + 2\langle  \nabla L(w_t;D) , A_t \rangle+\norm{A_t}^2},
\end{align*}
where we define $A_t = \Delta_t - \nabla L(w_t;D)$.

By the convexity and smoothness, we have
\begin{equation}
\norm{\nabla L(w_t;D )}^2 \le K \langle   \nabla L(w_t;D), w_t -w^* \rangle.\nonumber
\end{equation}

Combining these, for all $t \in {1,\ldots,T}$,
\begin{align*}
&~~~\norm{w_{t+1}-w^*}^2 \\
& \le \norm{w_t  -w^* }^2 -(2\eta - K\eta^2) \langle \nabla L(w_t;D), w_t -w^* \rangle - 2\eta \langle w_t - \eta \nabla L(w_t;D)-w^*, A_t  \rangle+\eta^2 \norm{A_t}^2\\
& \le \norm{w_t  -w^* }^2 -(2\eta - K\eta^2) \Paren{ L(w_t;D) -  L(w^*;D)} - 2\eta \langle w_t - \eta \nabla L(w_t;D)-w^*, A_t  \rangle+\eta^2 \norm{A_t}^2,
\end{align*}
where the last inequality uses the convexity and the fact  that $\eta\le \frac{2}{K}$.

Summing up the above inequalities and re-arranging the terms, we have 
\begin{align*}
&~~~~\Paren{2\eta - K\eta^2}\sum_{t=1}^{T} \Paren{L(w_t;D)- L(w^*;D)} \\
&\le \norm{w_1-w^*}^2 -  \norm{w_{T+1}-w^*}^2 -2\eta\sum_{t=1}^{T} \langle w_t - \eta \nabla L(w_t;D)-w^*, A_t  \rangle + \sum_{t=1}^{T}\eta^2\norm{A_t}^2\\
&\le D_w^2 -2\eta\sum_{t=1}^{T} \langle w_t - \eta \nabla L(w_t;D)-w^*, A_t  \rangle + \sum_{t=1}^{T}\eta^2\norm{A_t}^2.
\end{align*}

Taking the expectation on both sides, we have 

\begin{align*}
&~~~\expectation{\frac1T\sum_{t=1}^{T} \Paren{L(w_t;D)- L(w^*;D)}} \\
&\le \frac1{\Paren{2\eta - K\eta^2}T} \Paren{ D_w^2 - 2\eta \cdot \expectation{\sum_{t=1}^{T} \langle w_t - \eta \nabla L(w_t;D)-w^*, A_t  \rangle}+ \expectation{ \sum_{t=1}^{T}\eta^2\norm{A_t}^2}}.
\end{align*}
We first bound the third term in the parenthesis. According to the definition of $\sigma$ in Lemma~\ref{lem:theer}, we have
\begin{equation}
\label{equ:thi}
\expectation{ \sum_{t=1}^{T}\eta^2\norm{A_t}^2} \le T\eta^2\sigma^2.
\end{equation}

With respect to the second term,
\begin{align*}
&~~~~\expectation{\sum_{t=1}^{T} \langle w_t - \eta \nabla L(w_t;D)-w^*, A_t  \rangle} \nonumber\\
&= \expectation{\sum_{t=1}^{T} \langle w_t - \eta \nabla L(w_t;D)-w^*,  (\Delta_t - \nabla_t) +(\nabla_t - \nabla L(w_t;D) )\rangle}\nonumber\\
%&= \expectation{\sum_{t=1}^{T} \langle w_t - \eta \nabla L(w_t;D)-w^*,  \Delta_t - \nabla_t \rangle} + \expectation{\sum_{t=1}^{T} \langle w_t - \eta \nabla L(w_t;D)-w^*, \nabla_t - \nabla L(w_t;D) \rangle} \\
%&= \expectation{\sum_{t=1}^{T} \langle w_t - \eta \nabla L(w_t;D)-w^*,  \Delta_t - \nabla_t \rangle}+ \sum_{t=1}^{T}  \expectation{ \langle w_t - \eta \nabla L(w_t;D)-w^*,  \nabla_t - \nabla L(w_t;D) \rangle}\nonumber\\
&= \expectation{\sum_{t=1}^{T} \langle w_t - \eta \nabla L(w_t;D)-w^*,  \Delta_t - \nabla_t \rangle}+ \sum_{t=1}^{T}  \expectation{ \expectation{ \langle w_t - \eta \nabla L(w_t;D)-w^*,  \nabla_t - \nabla L(w_t;D) \rangle \mid w_t = w}}.\nonumber\\
\end{align*}
Note that 
\begin{align*}
&~~~~\expectation{ \langle w_t - \eta \nabla L(w_t;D)-w^*,  \nabla_t - \nabla L(w_t;D) \rangle \mid w_t = w}\\
& =  \langle w - \eta \nabla L(w;D)-w^*,  \expectation{ \nabla_t - \nabla L(w_t;D)  \mid w_t = w} \rangle = 0.
\end{align*}
where the last equality comes from the fact that $\nabla_t$ is an unbiased estimator of $ \nabla L(w_t;D)$.

Therefore,
\begin{equation}
\expectation{\sum_{t=1}^{T} \langle w_t - \eta \nabla L(w_t;D)-w^*, A_t  \rangle} = \expectation{\sum_{t=1}^{T} \langle w_t - \eta \nabla L(w_t;D)-w^*,  \Delta_t - \nabla_t \rangle}.\nonumber
\end{equation}
By Cauchy-Schwarz inequality, and the definitions in Lemma~\ref{lem:theer},
\begin{align}
\label{equ:sec}
&~~~\expectation{\sum_{t=1}^{T} \langle w_t - \eta \nabla L(w_t;D)-w^*,  \Delta_t - \nabla_t \rangle}\nonumber\\
&\le \expectation{ \sum_{t=1}^T \norm{w_t - \eta \nabla L(w_t;D)-w^*} \cdot \norm{\Delta_t - \nabla_t}}\nonumber\\
&\le T\Paren{D_w B+ \eta GB }.
\end{align}

Finally, by combining Equation~\eqref{equ:thi} and~\eqref{equ:sec}, and note that $K\eta\le 1$,
\begin{align*}
&~~~\expectation{\frac1T\sum_{t=1}^{T} \Paren{L(w_t;D)- L(w^*;D)}} \\
&\le \frac{D_w^2+ T\eta^2\sigma^2}{T\eta \cdot (2-K\eta)}+\frac{2 D_w B \eta+ 2\eta^2GB} {\eta \cdot (2-K\eta)}\\
&\le \frac{D_w^2+ T\eta^2\sigma^2}{T\eta}+\frac{2 D_w B \eta+ 2\eta^2GB} {\eta},
\end{align*}

By taking $\eta = \min \Paren{\frac1{K}, \frac{D_w}{\sqrt{T}\sigma}}$, we have $\frac{D_w^2}{T\eta} \le \frac{D_w^2 K}{T}+\frac{D_w\sigma}{\sqrt{T}}$, $\eta\sigma^2 \le \frac{D_w\sigma}{\sqrt{T}} $, and $G B \eta  \le \frac{GBD_w}{\sigma\sqrt{T}}$.
Therefore, by combining them,
\begin{align*}
&~~~\expectation{\frac1T\sum_{t=1}^{T} \Paren{L(w_t;D)- L(w^*;D)}} \le \frac{D_w^2 K}{T} + \frac{D_w \sigma}{\sqrt{T}}+ 2B D_w\Paren{1 + \frac{G}{\sigma \sqrt{T}}}.
\end{align*}

\end{document}